\pdfoutput=1

\documentclass[10pt,twocolumn,twoside,journal]{IEEEtran}

\usepackage{xspace}
\usepackage{graphicx}
\usepackage{amsmath}
\usepackage{amssymb}

\usepackage{amsthm}

\usepackage{verbatim}

\usepackage{cite,eucal}

\providecommand{\newoperator}[3]{%
\newcommand*{#1}{\mathop{#2}#3}}
\def\tsum{{\textstyle \sum}}
\def\SS{{\cal Q}}

\newoperator{\sst}{\mathrm{s.t.}\!:}{\nolimits}
\newoperator{\argmax}{\mathrm{argmax}}{\limits}

\def\T{{\!\top}}

\def\L{{\mathbf L}}
\def\Z{{\mathbf Z}}
\def\A{{\mathbf A}}
\def\B{{\mathbf B}}
\def\U{{\mathbf U}}
\def\C{{\mathbf C}}

\def\bSigma{{\boldsymbol \Sigma}}

\def\half{{\tfrac{1}{2}}}

\def\X{{\mathbf X}}
\def\Y{{\mathbf Y}}

\def\etal{\emph{et al.}\xspace}
\def\dist{{\bf dist}}
\def\ba{{\mathbf a}}
\def\innerp#1#2{{\left<#1, #2 \right>}}
\def\PSD{\succcurlyeq}
\def\vecGTE{\geq}
\def\NSD{\preccurlyeq}
\def\vecLTE{\leq}
\def\trace{\operatorname{\bf  Tr}}

\def\diag{\operatorname{\bf  diag}}

\def\ba{{\boldsymbol a}}
\def\psd{{p.s.d.}\xspace}

\newcommand{\bx}{ \boldsymbol{x} }
\newcommand{\bb}{ \boldsymbol{b} }

\newcommand{\bu}{ \boldsymbol{u} }

\newcommand{\real}{\ensuremath{ \mathbb{ R } }}
\newcommand{\symmetric}{\ensuremath{ \mathbb{ S } }}
\newcommand{\semidefinite}{\ensuremath{ \mathbb{ S }_+ }}

\newcommand{\normf}[1]{\ensuremath{ \bigl\| #1 \bigr\|_{ \mathrm{F} }^2 } }

\newcommand{\iprod}[2]{\ensuremath{ \left< #1, #2 \right> } }
\newcommand{\st}{{\ensuremath{\rm \; s.t.}\;}}

\usepackage{xspace,colortbl,multirow}
\usepackage{color}
\definecolor{tabgrey}{rgb}{0.,0.08,0.08}
\usepackage{url}

\newtheorem{thm}{Theorem}

\renewcommand{\comment}[1]{}

\begin{document}

\title{
        An Efficient Dual Approach to \\
        Distance Metric Learning
      }

\author{
         Chunhua Shen,
         Junae Kim,
         Fayao Liu,
         Lei Wang,
         Anton van den Hengel
\thanks
{
    The participation of C. Shen and A. van den Hengel was in part
    support by ARC grant LP120200485. C. Shen was also supported by
    ARC Future Fellowship FT120100969.
}
\thanks
{
    C. Shen, F. Liu, and A. van den Hengel are with Australian Center for
    Visual Technologies, and School of Computer Science at The
    University of Adelaide, SA 5005, Australia (e-mail:
    \{chunhua.shen, fayao.liu, anton.vandenhengel\}@adelaide.edu.au).
    Correspondence should be addressed to C. Shen.
}
\thanks
{
    J. Kim is with NICTA, Canberra Research Laboratory, ACT 2600,
    Australia (e-mail: junae.kim@gmail.com).
}
\thanks
{
    L. Wang is with University of Wollongong, NSW 2522, Australia
    (e-mail: leiw@uow.edu.au).
}
}

\maketitle

\begin{abstract}

    Distance metric learning is of fundamental interest in machine learning
    because the distance metric employed can significantly affect the
    performance of many learning methods. 
    Quadratic Mahalanobis metric  learning is  a popular approach  to the
    problem,  but typically  requires solving  a semidefinite  programming (SDP)  problem, which  is
    computationally  expensive.
    Standard interior-point SDP solvers typically have a complexity of $ O( D^{6.5} )$
    (with $  D $ the  dimension of input  data), and can thus only practically 
    solve problems exhibiting less than a few thousand
    variables. Since the number of variables is $ D (  D+1 ) / 2 $, this implies a limit upon
    the size of problem that can practically be solved of  around a
    few hundred dimensions.
    The complexity of the popular
    quadratic Mahalanobis metric learning  approach thus limits the size of  problem to which metric
    learning can  be applied.  Here we  propose a
    significantly more efficient  approach to  the metric
    learning  problem  based on  the  Lagrange  dual formulation of the problem.  The  proposed  formulation is  much
    simpler to implement, and therefore allows much  larger Mahalanobis metric  learning problems to  be solved.
    The time complexity of the proposed method is $ O (D ^  3 ) $, which is significantly lower than that of the SDP approach.
    Experiments  on a variety of datasets demonstrate that the
    proposed method   achieves an accuracy comparable to the
    state-of-the-art, but is applicable to significantly larger
    problems.  We also show that the proposed method can be applied
    to  solve more general Frobenius-norm regularized SDP
    problems approximately.

\end{abstract}


\tableofcontents 
\clearpage

\section{Introduction}

    Distance metric learning has attracted a lot of research
    interest recently in the machine learning and pattern recognition
    community due to its wide applications in various areas
    \cite{Weinberger2009JMLR,Shen2010TNN,xing2002distance,Domeniconi2005}.
    Methods relying upon the
    identification of an appropriate data-dependent distance
    metric 
    have been applied
    to a range of problems, from image
    classification and object recognition, to the analysis of genomes.
    The performance of many classic algorithms such as $ k $-nearest neighbor ($ k$NN)
    and $ k $-means clustering depends critically upon the distance metric employed.
   
    Large-margin metric learning  is an approach which
    focuses on identifying a metric by which the data points within
    the same class lie close to each other
    and those in different classes are separated by a large margin.
    Weinberger \etal's large-margin nearest neighbor (LMNN) \cite{Weinberger2009JMLR} 
    is a seminal    work illustrating the approach whereby the metric takes 
    the form of a Mahanalobis distance.
    Given input data $ \ba \in \real^D $,
    this approach to the metric learning problem can be framed as that of 
    learning the linear transformation $\L$
    which optimizes a criterion expressed in terms of Euclidean distances amongst 
    the projected data $ \L \ba \in \real^d$. 

    In order to obtain a convex problem,
    instead of learning the projection matrix
    ($\L \in \real^{ D \times d } $),
    one usually optimizes over the quadratic
    product of the projection matrix ($ \X = \L \L^\T $)
    \cite{xing2002distance,Weinberger2009JMLR}.
    This linearization {\em convexifies} the  original non-convex problem.
    The projection matrix may then be recovered by an
    eigen-decomposition or Cholesky decomposition of $ \X $. 
      
    Typical methods that learn the projection matrix $ \L $ are most
    of the spectral dimensionality reduction methods such as principle
    component analysis (PCA), Fisher linear discriminant analysis
    (LDA); and also neighborhood
    component analysis (NCA) \cite{goldberger2004neighbourhood},
    relevant component analysis (RCA)
    \cite{Shental2002}.
    Goldberger
    \etal  showed that NCA may outperform traditional dimensionality reduction
    methods \cite{goldberger2004neighbourhood}.
    NCA learns the projection matrix directly
    through optimization of a non-convex objective function.
    NCA is thus prone to becoming trapped in  local optima, particularly
    when applied to high-dimensional problems.
    RCA \cite{Shental2002} is an unsupervised metric learning method.
    RCA does not maximize the distance between different classes, but
    minimizes the distance between data in Chunklets. Chunklets
    consist of data that come from the same (although unknown)
    class.

    More methods on the topic of large-margin metric learning actually
    learn $ \X $ directly since 
    Xing \etal \cite{xing2002distance} proposed a global
    distance metric learning approach using a convex optimization
    method. 
    Although the experiments in \cite{xing2002distance} show improved performance on
    clustering problems,
    this is not the case when the method is applied to
    most classification problems.
    Davis \etal \cite{Davis2007Info}
    proposed an information theoretic metric learning (ITML)
    approach to the problem.
    The closest work to ours may be LMNN \cite{Weinberger2009JMLR}
    and BoostMetric \cite{BoostMetric}.
    LMNN is a Mahanalobis metric form of $k$-NN whereby
    the Mahanalobis metric is optimized such
    that the $k$-nearest neighbors
    are encouraged to belong to the same class while data points from different
    classes are separated by a large margin.
    The optimization take the form of an SDP problem. In order to improve the scalability of the
    algorithm, instead of using standard SDP solvers, Weinberger \etal
    \cite{Weinberger2009JMLR} proposed an alternating estimation and projection method.
    At each iteration, the updated estimate $ \X $ is projected back to the
    semidefinite cone using eigen-decomposition,
    in order to preserve the semi-definiteness of $ \X $.
    In this sense, at each iteration,
    the computational complexity of their algorithm is similar
    to that of ours. However, the alternating method needs an extremely large number of
    iterations to converge (the default value being
    $10,000$ in the authors' implementation).
    In contrast, our algorithm solves the corresponding Lagrange dual problem
    and needs only $ 20 \sim 30 $ iterations in most cases.
    In addition, the algorithm we propose is significantly easier to implement.
 
    As pointed in these earlier work, the disadvantage of solving for
    $\X$    is that one needs
    to solve a semidefinite programming (SDP) problem since $ \X $
    must be positive semidefinite (\psd).  Conventional interior-point
    SDP solvers have a  computation complexity of $ O( D^{ 6.5 } ) $,
    with $ D $ the dimension of input data. This high complexity
    hampers the application of metric learning to high-dimensional
    problems.
   
    To tackle this problem,
    here we propose here a new formulation of quadratic
    Mahalanobis metric learning using proximity comparison information
    and Frobenius norm regularization.
    The main contribution is that, with the proposed formulation,
    we can very efficiently solve the SDP problem in the dual space.
    Because strong duality holds, we can then recover the primal
    variable $ \X $ from the dual solution. The computational complexity
    of the optimization is dominated by eigen-decomposition, which is
    $ O( D^{ 3 } ) $ and thus the overall complexity is  $ O( t \cdot D^{ 3 } ) $, 
    where $ t $ is the number of iterations required for convergence.  Note that $t$ 
    does not depend on the 
    size of the data, and is typically $ t \approx  20 \sim 30$.

    A number of methods exist in the literature for 
    large-scale \psd metric learning. 
    Shen \etal~\cite{BoostMetric,ShenJMLR}
    introduced BoostMetric
    by adapting the boosting technique, 
    typically applied to classification, to distance metric learning. 
    This work exploits an important theorem which shows
    that a positive semidefinite  matrix
    with trace of one can always be represented as a convex combination
    of multiple rank-one matrices.
    The work of Shen \etal generalized LPBoost \cite{LPBoost}
    and AdaBoost 
    by showing that it is possible to use matrices as weak learners within these algorithms, 
    in addition to the more traditional use of classifiers or
    regressors as weak learners.    
    The approach we propose here, FrobMetric, is inspired by BoostMetric in the sense that
    both algorithms use proximity comparisons between triplets
    as the source of the training information. 
    The critical distinction between FrobMetric and BoostMetric, however, is that 
    reformulating the problem to use the Frobenius regularization---rather than the trace
    norm regularization---allows the development of a dual
    form of the resulting optimization problem which may be solved far
    more efficiently.
    The BoostMetric approach  iteratively computes the squared
    Mahalanobis distance metric using a rank-one update at each
    iteration.  This has the advantage that only the leading
    eigenvector need be calculated, but leads to slower convergence.
    Indeed, for BoostMetric, the convergence rate remains unclear.      
	The proposed FrobMetric method, 
    in contrast, requires more calculations per iteration, but
    converges in significantly fewer iterations.
    Actually in our implementation, the convergence rate of FrobMetric
    is guaranteed by the employed Quasi-Newton method. 

    The main contributions of this work are as follows.
    \begin{enumerate}
    \item
    We propose a novel formulation of the metric learning problem,
    based on the application of Frobenius norm regularization.
    \item We develop a method for solving this formulation of the problem which is based on optimizing its Lagrange dual.  This method may be practically applied to a much more complex datasets than the competing SDP approach, as it scales better to large databases and to high dimensional data.

    \item
     We generalize the method such that it may be used to solve any Frobenius norm regularized
     SDP problem.  Such problems have many applications in machine learning and computer vision, and by way of example we show that it may be used to
     approximately solve the Frobenius norm perturbed maximum variance
     unfolding (MVU) problem \cite{mvu}. We demonstrate that the proposed
     method is considerably more efficient
     than the original MVU implementation on a variety of data sets
     and that a plausible embedding is obtained. %

    \end{enumerate}

    The proposed scalable semidefinite optimization method
    can be viewed as an extension of the work of Boyd and Xiao \cite{LS2005}.
    The subject of Boyd and Xiao  in \cite{LS2005} was similarly a
    semidefinite least-squares problem: finding the
    covariance matrix that is closest to a given matrix
    under the Frobenius norm metric. Here we study the large-margin
    Mahalanobis metric learning problem, where, in contrast, the objective function is
    not a least squares fitting problem. We also discuss, in Section~\ref{sec:general} the 
    application of the proposed approach to general SDP problems which have Frobenius norm regularization terms.
	Note also that a precursor to the approach described here also appeared in \cite{Shen2011CVPRb}.
    Here we have provided more theoretical analysis as well as experimental
    results. 

    In summary, we propose a simple,
    efficient and scalable optimization method for
    quadratic Mahalanobis metric learning. The formulated optimization
    problem is  convex, thus guaranteeing that the global optimum can be attained in
    polynomial time \cite{boyd2004convex}. Moreover, by working with the Lagrange dual problem,
    we are able to use off-the-shelf eigen-decomposition
    and gradient descent methods such as L-BFGS-B to solve
    the problem.

    \subsection{Notation}
    A column vector is denoted by a bold lower-case letter ($\bx$)
    and a matrix is by a bold upper-case letter ($\X$).
    The fact that a matrix $\A$ is 
    positive semidefinite (\psd) is denoted thus $ \A \PSD 0$.  The
    inequality $\A \PSD \B$ is intended to indicate that $\A - \B \PSD
    0$.
    In the case of vectors, $ \ba \vecGTE \bb  $ denotes the
    element-wise version of the inequality, and when applied relative
    to a scalar ({\em e.g.}, $\ba \vecGTE 0$ ) the inequality is intended to
    apply to every element of the vector.
    For matrices, we denote by $ \real ^{ m \times n} $ the vector space
    of real matrices of size $ m \times n $, and 
    the space of real 
    symmetric
    matrices as $ \mathbb S$.  Similarly,
    the space of symmetric matrices of size
    $ n \times n $ is $ \symmetric^n$, and
    the space of symmetric positive semidefinite matrices
    of size $ n \times n $ is denoted as
    $ \semidefinite^n $.
    The inner product defined on these spaces is $ \iprod{ \A }{ \B } = \trace( \A^\T \B ) $.
    Here $\trace(\cdot)$ calculates the trace of a matrix.
    The Frobenius norm of a matrix is defined as
    $ \normf{\X} = \trace( \X \X ^\T ) = \trace( \X ^\T \X  )$,
    which is the sum of all the squared elements of $ \X $.
    $ \diag( \cdot ) $ extracts the diagonal elements of a square matrix.
    Given a symmetric matrix $ \X $ and its eigen-decomposition
    $ \X = \U \bSigma \U^\T  $ ($ \U$ being an orthonormal matrix,
    and $ \bSigma$ being real and diagonal),
    we define the positive part of $ \X $ as
    \begin{equation*}
        ( \X )_+ = \U  
                     \Bigl[ 
                                \max (  \diag(\bSigma), 0 ) 
                     \Bigr]     \U^\T,
    \end{equation*}
    and the negative part of $ \X $ as
    \begin{equation*}
        ( \X )_- = \U   \Bigl[ 
                                \min (  \diag(\bSigma), 0 ) 
                        \Bigr] \U^\T.
    \end{equation*}
    Clearly,
    $ \X= (\X)_+ + (\X)_- $ holds.

    \subsection{Euclidean  Projection onto the \psd cone} 
    Our proposed method relies on the following standard results,
    which can be found in textbooks such as Chapter 8 of \cite{boyd2004convex}.  
    The positive semidefinite part $ (\X)_+ $ of $ \X $ is the
    projection of $ \X $ onto the \psd cone:
    \begin{equation}
        \label{EQ:proj_psd}
        ( \X)_+ = \left\{
                \min_\Y  \;   \normf{ \Y - \X  },  \st \Y \PSD
        0     \right\}.  
    \end{equation}
    It is not difficult to check that,
    for any $ \Y \PSD 0 $, 
    \[
    \normf{   \X -    ( \X)_+   }  =  \normf{   ( \X)_- }   \leq 
    \normf{   \X - \Y  }. 
    \]
    In other words, although the optimization  problem in 
    \eqref{EQ:proj_psd} appears as an
    SDP programming problem, it can be simply solved by using
    eigen-decomposition, which is efficient. 
    It is this key observation that serves as the backbone of the
    proposed fast method. 

    The rest of the paper is organized as follows.
    In Section \ref{sec:alg}, we present the main algorithm for
    learning a Mahalanobis metric using efficient optimization. In
    Section \ref{sec:general}, we extend our algorithm to more general
    Frobenius norm regularized semidefinite problems.
    The experiments on various datasets are in shown
    Section \ref{Experiments} and we conclude the paper in Section
    \ref{Conclusion}.

\section{Large-margin Distance Metric Learning}
\label{sec:alg}

    We now briefly review quadratic Mahalanobis distance metrics. Suppose that we
    have a set of triplets $\SS = \{(\ba_i, \ba_j, \ba_k) \}$, which encodes proximity comparison
    information. Suppose also that $\dist_{ij}$ computes the Mahalanobis distance between $\ba_i$ and $\ba_j$ under a
    proper Mahalanobis matrix. That is, $\dist_{ij} = \Vert{\ba_i - \ba_j } \Vert_{\X}^2 = (  \ba_i
    - \ba_j) ^\T \X (  \ba_i - \ba_j  )$, where 
    $\X \in \semidefinite^{D\times D}$, 
    is positive
    semidefinite. Such a Mahanalobis metric may equally be parameterized by a projection matrix 
    $\L$ where $\X = \L\L^\T$.

    Let us define the margin associated with a training triplet as
            $\rho_r = (  \ba_i - \ba_k  ) ^\T \X
            (  \ba_i - \ba_k  ) - (  \ba_i - \ba_j  ) ^\T \X
            (  \ba_i - \ba_j  ) = \innerp{\A_r}{\X},
            $ with $\A_r = (  \ba_i - \ba_k  )
            (  \ba_i - \ba_k  )^\T  - (  \ba_i - \ba_j  )
            (  \ba_i - \ba_j  )^\T $.
            Here $ r $ represents the index of the current triplet within the set of $m$ training triplets $ \SS $.
As will be shown in the experiments below, this type of proximity comparison among triplets
    may be easier to obtain that explicit distances for some applications like image retrieval.  Here the metric
learning procedure solely relies on the matrices $  \A_r $ ($r = 1, \dots, m$).

\subsection{Primal problems of Mahanalobis metric learning}

Putting it into the large-margin learning framework, the
optimization problem is to maximize the margin with a regularization
term that is intended to avoid over-fitting (or, in some cases, makes the problem well-posed):
     \begin{align}
      \label{PR:1}
         \max_{\X, \rho, \boldsymbol \xi} \, & \rho
                    -
                      \tfrac{C_1}{m} \tsum_{r=1}^{m}  \xi_r
        \notag \\
        \sst \, & \innerp{\A_r}{\X} \geq \rho - \xi_r, r =
        1,\cdots,m, \tag{P1}
        \\
        \, & {\boldsymbol \xi} \vecGTE 0,
        \rho \vecGTE 0,
        \trace(\X) = 1, \X \PSD 0.
       \notag
      \end{align}
Here $\trace(\X) = 1$ removes the scale ambiguity in $\X$. This is the
formulation proposed in BoostMetric \cite{BoostMetric}.
We can write the above problem equivalently as
     \begin{align}
      \label{PR:2}
         \min_{\X, \boldsymbol \xi} \, & \trace(\X)
                    +
                      \tfrac{C_2}{m} \tsum_{r=1}^{m}  \xi_r
        \notag
        \\
        \sst \, & \innerp{\A_r}{\X} \geq 1 - \xi_r, r =
        1,\cdots,m, \tag{P2}
        \\
        \, & {\boldsymbol \xi} \vecGTE 0, \X \PSD 0.
       \notag
      \end{align}
These  formulations are exactly equivalent given the appropriate choice of the trade-off
parameters $C_1$ and $ C_2 $.
The theorem is as follows.
      \begin{thm}
          A solution of \eqref{PR:1}, $ \X^\star $, is also a solution of
          \eqref{PR:2} and
          vice versa up to a scale factor.

          More precisely, if \eqref{PR:1} with parameter $ C_1 $
          has a solution $( \X^\star$,
          $ {\boldsymbol \xi}^\star $, $ \rho ^ \star > 0 )$, then
          $ ( \tfrac{\X^\star}{\rho^\star},
              \tfrac{ {\boldsymbol \xi}^\star }{\rho^\star} )$
             is the solution of \eqref{PR:2} with parameter
             $ C_2 = C_1/ { {\rm Opt }\eqref{PR:1}  }   $.
             Here $ { {\rm Opt }\eqref{PR:1}  } $
             is the optimal objective
             value of \eqref{PR:1}.
     \end{thm}
\begin{proof}
    It is well known that
    the necessary and sufficient conditions for the optimality of SDP problems are primal feasibility, dual
    feasibility, and equality of the primal and dual objectives.
    We can easily derive the duals of \eqref{PR:1} and \eqref{PR:2} respectively:
    \begin{align}
        \label{EQ:Dual1}
        \min_{\gamma, \bu} \; & \gamma
        \notag
        \\
        \sst &  \textstyle \sum_{r=1}^m u_r \A_r  \NSD \gamma {\bf I},
        \tag{D1}
        \\
        & {\bf 1}^\T \bu = 1; 0 \vecLTE \bu \vecLTE \tfrac{C_1}{m}; \notag
    \end{align}
    and,
    \begin{align}
        \label{EQ:Dual2}
        \max_{\bu} \; & \textstyle \sum_{r=1}^m u_r
        \notag
        \\
        \sst &  \textstyle \sum_{r=1}^m u_r \A_r  \NSD {\bf I},
        \tag{D2}
        \\
        &  0 \vecLTE \bu \vecLTE \tfrac{C_2}{m}. \notag
    \end{align}
    Here $ \bf I $ is the identity matrix.

    Let $  ( \X^\star$,
    $ {\boldsymbol \xi}^\star $, $ \rho ^ \star  )$ represent the optimum of the 
    primal problem \eqref{PR:1}. 
    Primal feasibility of
    \eqref{PR:1} implies primal feasibility of \eqref{PR:2}, and thus that
    \[ \innerp{ \A_r  }{ \tfrac{\X^\star }{ \rho^\star} } \geq 1 - \tfrac{\xi^\star }{ \rho^\star } .\]

    Let $ ( \gamma^\star, \bu^\star  ) $ be the optimal solution of the dual problem
    \eqref{EQ:Dual1}. Dual feasibility of \eqref{EQ:Dual1} implies dual feasibility
    of \eqref{EQ:Dual2}, and thus that $  \sum_r { u_r^\star } / { \gamma^\star } \A_r \NSD \bf I $,
    and $  0 \vecLTE \bu^\star/ { \gamma^\star } \vecLTE C_2 / ( m { \gamma^\star } ) $.
    Since the duality gap between \eqref{PR:1} and \eqref{EQ:Dual1} is zero,
    $ {\rm Opt}\eqref{EQ:Dual1} = \gamma^\star = {\rm Opt}\eqref{PR:1} $.

    Last we need to show that the objective function values
    of \eqref{PR:2} and \eqref{EQ:Dual2} are the same. This is easy to verify
    from the fact that ${\rm Opt}\eqref{PR:1} = {\rm Opt}\eqref{EQ:Dual1}  $:
		\begin{align*}
   \rho^\star - \tfrac{C_1}{m} {\bf 1}^\T {\boldsymbol \xi}
    & = \gamma^\star \\
    \Longrightarrow 
      \rho^\star  ({\bf 1}^\T \bu^\star ) - \tfrac{C_1}{m} {\bf 1}^\T {\boldsymbol \xi}
    & = \gamma^\star \trace( \X^\star)  \\
    \Longrightarrow 
     \trace{ ( \X^\star )  /  \rho^\star  } + \tfrac{C_2}{m}
    {\bf 1}^\T{\boldsymbol \xi^\star } / \rho^\star  & =
    ({\bf 1}^\T \bu^\star ) /
    \gamma^\star.
		\end{align*}
        This concludes the proof. 
\end{proof}

    Both problems can be written in the form of standard SDP problems since
    the objective function is linear and a \psd constraint is
    involved.  Recall that we are interested in a Frobenius norm
    regularization rather that a trace norm regularization. The key
    observation is that {\em the Frobenius norm regularization term
    leads to a simple and scalable optimization}. So replacing the
    trace norm in \eqref{PR:2} with the Frobenius norm we have:
    \begin{align}
    \label{PR:4}
      \min_{\X, \xi} \; & \tfrac{1}{2} \normf{\X}
                + \tfrac{C_3}{m} \tsum_{r=1}^{m}  \xi_r
        \notag
        \\
        \sst \; & \innerp{\A_r}{\X} \geq 1 - \xi_r, r =
        1,\cdots,m, \tag{P3}
        \\
        \; & { \boldsymbol \xi } \vecGTE 0, \X \PSD 0.
       \notag
    \end{align}
    Although \eqref{PR:4} and \eqref{PR:2}  are not exactly the same,
    the only difference is the regularization term.
    Different regularizations can lead to different solutions.
    However, as the $ \ell_1 $ and $ \ell_2 $ norm regularizations 
    in the case of vector variables such as in support vector machines
    (SVM), 
    in general, these two regularizations would perform similarly in
    terms of the final classification accuracy. 
    Here, one does not expect that a particular form of
    regularization, either the trace  or Frobenius norm
    regularization, would perform better than the other one.  
    As we have pointed out, 
    the advantage of the Frobenius norm is faster optimization.  

    One may convert \eqref{PR:4} into a standard SDP problem
    by introducing an auxiliary variable:
     \begin{align*}
    \label{PR:4B}
      \min_{\X, \xi, \delta } \; &  \delta
                + \tfrac{C_3}{m} \tsum_{r=1}^{m}  \xi_r
        \notag
        \\
        \sst \; & \innerp{\A_r}{\X} \geq 1 - \xi_r, r =
        1,\cdots,m, 
        \\
        \; & { \boldsymbol \xi } \vecGTE 0, \X \PSD 0;
       \;
       \tfrac{1}{2} \normf{\X} \leq \delta. 
       \notag
    \end{align*}
    The last constraint can be formulated as a \psd constraint
    $  
    \begin{bmatrix}
        1  &  \X \\
        \X & 2\delta
    \end{bmatrix} \PSD 0.
    $
    So in theory, we can  use an
    off-the-shelf SDP solver to solve this primal problem directly.
    However, as mentioned previously, the computational complexity of this approach is very
    high, meaning that only small-scale problems can be solved within reasonable
    CPU time limits.  
    
    Next, we show that, the Lagrange dual problem of
    \eqref{PR:4} has some desirable properties.
    
    \subsection{Dual problems and desirable properties} 
    \label{SSUBSEC:Dual}

		We first introduce the Lagrangian dual multipliers, $ \Z $ 
 		which we associate with the \psd constraint $ \X \PSD 0$, and  
 		$ \bu $ which we associate with the remaining constraints upon ${ \boldsymbol \xi }$.

        The Lagrangian
        of \eqref{PR:4} then becomes
        \begin{align*}
            \ell ( \underbrace{ \X, {\boldsymbol \xi} }_{\rm primal },
            \underbrace{ \Z, \bu,  {\boldsymbol p} }_{\rm dual }) =
                \half \Vert \X \Vert_{\rm F}^2 + \tfrac{C_3}{m} \tsum_{r=1}^m \xi_r
                - \tsum_r { u_r \left< \A_r, \X \right> }
                \\
                + \tsum_r u_r - \tsum_r u_r \xi_r
                - {\boldsymbol p} ^\T {\boldsymbol \xi} - \left< \X, \Z    \right>
        \end{align*}
        with $ \bu \vecGTE 0$ and $ \Z \PSD 0$.
        We need to minimize the Lagrangian over $ \X $ and $ \boldsymbol
        \xi$, which can be done by setting the first derivative to zero, from which we see that
        \begin{equation}
            \label{EQ:KKT1}
            \X^\star = \Z^\star + \tsum_r u_r^\star \A_r,
        \end{equation}
        and
        $  \tfrac{C_3 }{ m } \vecGTE \bu \vecGTE 0 $.
        Substituting the expression for $ \X $
        back into the Lagrangian, we obtain the dual formulation:
     \begin{align}
      \label{PR:5}
         \max_{\Z, \bu} \; & \tsum_{r=1}^{m}  u_r
                    -
                      \tfrac{1}{2}  \normf{\Z + \tsum_{r=1}^{m} u_r\A_r}
        \tag{D3}
        \\
        \sst \; & \tfrac{C_3}{m} \vecGTE \bu \vecGTE 0, \Z \PSD 0.
       \notag
      \end{align}
        This dual problem still has a \psd constraint and it is
        not clear how it may be solved more efficiently than by using standard
        interior-point methods.
        Note, however, that as both the primal and dual problems are convex, 
        Slater's condition holds, under mild conditions (see \cite{boyd2004convex} for details).  
        Strong duality thus holds between \eqref{PR:4} and
        \eqref{PR:5}, %
        which means that the 
        objective values of these two problem coincide at optimality
        and in many cases
        we are able to indirectly solve the primal by solving the dual and vice versa.
        The  Karush--Kuhn--Tucker  (KKT) conditions \eqref{EQ:KKT1} thus
        enable us to recover $ \X^\star $, which is the primal
        variable of interest, from the dual solution.

        Given a fixed $\bu$, the dual problem \eqref{PR:5} may be simplified
        \begin{equation} \label{EQ:D1}
        \min_{\Z} \; \normf{\Z + \tsum_{r=1}^{m} u_r\A_r}, ~\sst \Z \PSD 0.
        \end{equation}
    To simplify the
    notation we define $ \hat{\A} $ as a function of $ \bu $
    \[
            \hat{\A} = -\tsum_{r=1}^{m}u_r\A_r.
    \]
        Problem \eqref{EQ:D1} then becomes that of finding the \psd~matrix  $ \Z $
    such that $  \Vert{  \Z -  \hat{\A} } \Vert^2_{\rm F} $
    is minimized.
    This problem has a closed-form solution, which is the positive part of
    $\hat{\A}$:
    \begin{equation}
        \label{EQ:D2}
        {\Z}^\star = (\hat{\A})_+.
    \end{equation}
     Now the original dual problem may be simplified 
     \begin{equation}
      \label{EQ:D3}
         \max_{\bu} \, \tsum_{r=1}^{m}  u_r
                    -
                    \tfrac{1}{2} \bigl\| {(\hat{\A})_-} \bigr\|^2_{\rm F},
                    \;
        \sst \, \tfrac{C_3}{m} \vecGTE \bu \vecGTE 0.
      \end{equation}
      The KKT condition is simplified into
      \begin{equation}
          \label{EQ:KKT2}
          \X^\star = ( \hat{\A} )_+  - \hat{\A} = - ( \hat{\A} )_-.
      \end{equation}
      From the definition of the operator $ ( \cdot)_-  $,
      $ \X^\star $ computed by \eqref{EQ:KKT2} must be \psd
      Note that we have now achieved a
      simplified dual problem which has no matrix variables, and only simple box constraints on $ \bu $.
      The fact that
      the objective function of \eqref{EQ:D3} is differentiable (but not
      twice
        differentiable)
      allows us to optimize for $ \bu $ in \eqref{EQ:D3}
      using gradient descent methods (see Sect. 5.2 in
      \cite{Borwein2000}).
    To illustrate why the objective function is differentiable, we can
    see the following simple example. 
    For $ F(\X) =  \tfrac{1}{2} \normf { (\X)_-  }  $, the gradient
    can be calculated as 
    \[
        \nabla F(\X) =   (\X)_-,
    \]
    because of the following fact.
    Given a symmetric $ \delta\X $, we have
    \[
            F(  \X + \delta\X ) = F( \X ) + \trace(  \delta\X   (\X)_-
            ) + o( \delta\X ).   
    \]
    This can be verified by using the perturbation theory of
    eigenvalues of symmetric matrices. When we set $ \delta\X $ to be
    very small, the above equality is the definition of gradient.  

    Hence, we can use a sophisticated off-the-shelf first-order Newton
    algorithm such as L-BFGS-B \cite{Liu1989} to solve \eqref{EQ:D3}.
    In summary,
    the optimization procedure is as follows.
    \begin{enumerate}
        \item 
            Input the training triplets and calculate $ \A_r$, $ r = 1 \dots m$.
        \item 
            Calculate the gradient of the objective function in \eqref{EQ:D3},
            and use L-BFGS-B to optimize \eqref{EQ:D3}.
        \item
            Calculate $ \hat\A$ using the output of L-BFGS-B (namely,
            $ \bu^\star$) and
            compute   $ \X^\star $ from \eqref{EQ:KKT2}
            using eigen-de\-com\-po\-si\-tion.
    \end{enumerate}
    To implement this approach, one only needs to implement  the callback
    function of L-BFGS-B, which computes the gradient of
    the objective function of \eqref{EQ:D3}.
    Note that other gradient methods such as  conjugate gradients may be
    preferred when the number of constraints (i.e., the size of
    training triplet set,  $ m$) is large.
    The gradient of dual problem \eqref{EQ:D3} can be calculated as
    \[
        g ( u_r ) = 1 + \innerp{ (  \hat\A  )_-   } { \A_r },
        r = 1, \dots, m.
    \]
    So, at each iteration,
    the computation of $ (  \hat\A  )_- $, which requires
    full eigen-decomposition,
    only need be calculated once in order to evaluate all of the gradients, as
    well as the function value.
    When the number of constraints is not far more than the
    dimensionality of the data, eigen-decomposition dominates the
    computational complexity at each iteration. In this case, the overall
    complexity is $ O ( t \cdot D^3 ) $ with $ t $ being
    around $ 20 \sim 30 $.

\section{General Frobenius Norm SDP}
\label{sec:general}

    In this section, we generalize the proposed idea
    to a broader setting.
    The general formulation of an SDP problem writes:
    \[
    \min_{\X} \; \iprod{\C}{\X}, \st \, \X \PSD 0,
                \iprod{\A_i} { \X } \leq b_i, i = 1\dots m.
    \]
    We consider its Frobenius norm regularized version:
    \[
    \min_{\X}  \; \iprod{\C}{\X} + \tfrac{1}{2\sigma}
                 \normf{\X}, \st \, \X \PSD 0,
                \iprod{\A_i} { \X }
                \leq b_i, \forall i.
    \]
    Here $ \sigma $ is a regularized constant.
    We start by deriving the Lagrange dual of this Frobenius norm regularized SDP.
    The  dual problem is,
    \begin{align}
        \label{EQ:D10}
        \min_{\Z,\bu}  \;   \tfrac{1}{2} \sigma \normf{\Z- \C - {\hat\A}} + \bb^\T \bu,
        \st \Z \PSD 0, \bu \vecGTE
        0.
    \end{align}
    The KKT condition is
    \begin{equation}
        \label{EQ:2A}
        \X^\star = \sigma ( \Z^\star - \hat\A  - \C ),
    \end{equation}
    where we have introduced the notation $ \hat\A = \sum_{i =1}^m u_i \A_i $.
    Keep it in mind that $ \hat\A $ is a function of the dual variable $ \bu $.
    As in the case of metric learning,
    the important observation is that $ \Z $ has an analytical solution when
    $ \bu $ is fixed:
    \begin{equation}
        \label{EQ:E1}
        \Z = ( \C + \hat\A  )_+.
    \end{equation}
    Therefore we can simplify \eqref{EQ:D10} into
    \begin{align}
        \label{EQ:D20}
        \min_{\bu} \; \tfrac{1}{2} \sigma \normf{ ( \C+\hat\A )_- }
         + \bb^\T \bu, \st \bu \vecGTE 0.
    \end{align}
    So now
    we can efficiently solve the dual problem using gradient descent methods.
    The gradient of the dual function is
    \[
    g ( u_i ) = \sigma \iprod{ ( \C + \hat\A )_- } { \A_i }
    + b_i, \forall i=1\dots m.
    \]
    At optimality, we have $ \X ^\star = - \sigma (  \C + {\hat\A}^\star )_- $.

    The core idea of the proposed method here may be applied to an SDP which
    has a term in the format of Frobenius norm, either in the objective function
    or in the constraints.

    In order to demonstrate the performance of the proposed general Frobenius
    norm SDP approach,
    we will show how it may be applied to the problem of Maximum Variance Unfolding (MVU).
    The MVU optimization problem writes
    \[
    \max_{\X} \, \trace( \X ) \, \st \innerp{ \A_i  }{ \X } \leq b_i, \forall i;
    {\boldsymbol 1}^\T \X {\boldsymbol 1} = 0; \X \PSD 0.
    \]
    Here $ \{ \A_i, b_i \}  $, $ i = 1\cdots$, encode the local distance constraints.
    This problem can be solved using off-the-shelf SDP solvers, which, as is described above, 
    does not scale well.
    Using the proposed approach, we modify the objective function to 
    $  \max_{\X} \, \trace( \X ) - \frac{1}{2\sigma} || \X ||^2_{\rm F}$. 
    When $ \sigma $ is sufficiently large,
    the solution to this Frobenius norm perturbed version is a reasonable
    approximation to the original problem.
    We thus use the proposed approach to solve MVU approximately.

\section{Experimental results}
\label{Experiments}

    We first run metric learning experiments on  UCI benchmark data,
    face recognition, and action recognition datasets.
    We then approximately solve the MVU problem
    \cite{mvu} using the proposed general Frobenius norm SDP approach.

\subsection{Distance metric learning}

\begin{table*}[tb!]
\begin{center}
\caption{Test errors of various metric learning methods
on  UCI data sets with $3$-NN. NCA \cite{goldberger2004neighbourhood}
does not output a result on those larger data sets due to memory problems.
Standard deviation is reported for data sets having multiple runs.
}
\label{Table:UCIResults} \centering
    \renewcommand{\arraystretch}{1.25}
\begin{tabular}{l||c|c|c|c|c|c|c}
\hline
                        & MNIST & USPS & letters & Yale faces &  Bal & Wine & Iris\\
\arrayrulecolor{tabgrey}
\hline\hline

\# {samples}   & 70,000 & 11,000 & 20,000 & 2,414 & 625  & 178 &
150\\\hline

\# triplets  & 450,000 & 69,300 & 94,500 & 15,210 & 3,942& 1,125&
945\\\hline

 dimension  & 784    & 256    & 16     & 1,024 & 4 & 13 & 4
\\\hline

 dimension after PCA & 164    & 60     &        & 300   &  & &
\\\hline

\# training & 50,000 & 7,700  & 10,500 & 1,690  & 438  & 125 &
105\\\hline

\# validation  & 10,000  & 1,650  & 4,500  & 362   & 94   & 27 &
23\\\hline

\# test & 10,000  & 1,650  & 5,000  & 362   &  93   & 26 &
22\\\hline

\# classes & 10     & 10     & 26     & 38    &  3    & 3 &
3\\\hline

\# runs  & 1      & 10      & 1     & 10    & 10    &  10  & 10 \\

\hline \hline

\textbf{Error Rates $ \%$} &  &  &  &  &   &  & \\ \hline
Euclidean & 3.19 &  4.78 (0.40) &  5.42 &  28.07 (2.07) &  18.60 (3.96) & 28.08 (7.49) & 3.64 (4.18)\\
\hline

PCA & 3.10 &  3.49 (0.62) & - &  28.65 (2.18) & - & - & - \\ \hline

LDA & 8.76 & 6.96 (0.68) & 4.44 &  5.08 (1.15) &  12.58 (2.38) & 0.77 (1.62) & \textbf{3.18 (3.07)}\\
\hline

SVM  & 2.97 & \textbf{2.15 (0.30)} &
2.96 &  \textbf{4.94 (2.14)} &  \textbf{5.59 (3.61)} & 1.15 (1.86) & 3.64 (3.59) \\
\hline

RCA \cite{Shental2002} & 7.85 & 5.35 (0.52) & 4.64 &  7.65 (1.08) &  17.42 (3.58) & \textbf{0.38 (1.22)} & \textbf{3.18 (3.07)}\\
\hline

NCA \cite{goldberger2004neighbourhood} & - & - & - & -  &  18.28 (3.58) & 28.08 (7.49) & \textbf{3.18 (3.74)}\\
\hline

LMNN \cite{Weinberger2009JMLR} & \textbf{2.30} &  3.49 (0.62) &  3.82
&  14.75
(12.11) & 12.04 (5.59) & 3.46 (3.82)\footnotemark & 3.64 (2.87)\\
\hline

ITML \cite{Davis2007Info} &   2.80 &  3.85 (1.13) &  7.20 &  19.39
(2.11) & 10.11 (4.06) & 28.46 (8.35) & 3.64 (3.59)
\\ \hline

BoostMetric \cite{BoostMetric} & 2.76 & 2.53 (0.47) &
3.06 &  6.91 (1.90) &  10.11 (3.45) & 3.08 (3.53) & \textbf{3.18 (3.74)} \\
\hline

FrobMetric (this work) & 2.56 & 2.32 (0.31) &
 \textbf{2.72} & 9.20 (1.06) &  9.68 (3.21) & 3.85 (4.44) & 3.64 (3.59) \\
\hline

\hline \hline

\textbf{Computational Time} &  &  &  &  &   &  & \\ \hline
LMNN & 11h &  20s &  1249s &  896s &  5s & 2s & 2s\\
\hline

ITML &  1479s & 72s  &  55s &  5970s &  8s & 4s & 4s\\
\hline

BoostMetric & 9.5h & 338s & \textbf{3s} &  572s &  \textbf{less than 1s} & 2s & \textbf{less than 1s}\\
\hline

FrobMetric & \textbf{280s} & \textbf{9s} & 13s &  \textbf{335s} &  \textbf{less than 1s} & \textbf{less than 1s} & \textbf{less than 1s}\\
\arrayrulecolor{black}

\hline
\end{tabular}
\end{center}
\end{table*}
\footnotetext{LMNN can solve for either $ \X $ or the projection
matrix $ \L $. 
When LMNN solves for $ \X $ on ``Wine''
set, the error rate is $20.77\% \, \pm 14.18\%$.
}

\subsubsection{UCI benchmark test}\label{uci}

    We perform a
    comparison between the proposed FrobMetric and a selection of the current state-of-the-art distance
    metric learning methods, including RCA \cite{Shental2002},
    NCA \cite{goldberger2004neighbourhood},
    LMNN \cite{Weinberger2009JMLR}, BoostMetric \cite{BoostMetric}
    and ITML \cite{Davis2007Info} on data
    sets from the UCI Repository.

    We have included results from PCA, LDA and support vector machine (SVM) with
    RBF Gaussian kernel as baseline approaches.
    The SVM results achieved using the {\tt libsvm} \cite{LibSVM}
    implementation. The kernel and regularization parameters
    of the SVMs were selected using cross validation.

    As in  \cite{Weinberger2009JMLR},
    for some data sets  (MNIST, Yale faces and USPS),
    we have applied PCA to reduce the original
    dimensionality and to reduce noise.

    For all experiments, the task is to classify unseen instances in a
    testing subset. To accumulate statistics, the data are randomly
    split into $10$ training/validating/testing subsets, except
    MNIST and Letter, which are already divided into subsets. We tuned
    the regularization parameter in the compared methods using cross-validation.
    In this experiment, about $15\%$ of data  are used  for cross-validation
    and $15\%$ for testing.

    For FrobMetric and BoostMetric in \cite{BoostMetric},
    we use $3$-nearest neighbors to generate triplets and check the
    performance using $3$NN.
    For each training sample ${\ba_i}$,
    we find its $3$ nearest neighbors in the same class 
    and the $3$ nearest neighbors in the difference classes.
    With $3$ nearest neighbors information, the number of triplets
    of each data set for FrobMetric and BoostMetric
    are shown in Table \ref{Table:UCIResults}. FrobMetric
    and BoostMetric have used exactly the same training information.
    Note that other methods do not
    use triplets as training data.
    The error rates based on $3$NN and computational time for each learning
    metric are shown as well.

    Experiment settings for LMNN and ITML follow the original work
    \cite{Weinberger2009JMLR} and \cite{Davis2007Info}, respectively.
    The identity matrix is used for ITML's initial metric matrix.
    For NCA, RCA, LMNN, ITML and BoostMetric, we used the codes provided
    by the authors.  We implement our FrobMetric in Matlab and L-BFGS-B
    is in Fortran and a Matlab interface is made.  All the computation time
    is reported on a workstation with 4 Intel Xeon E5520 (2.27GHz) CPUs
    (only single core is used) and 32 GB RAM.

    Table~\ref{Table:UCIResults} illustrates that the proposed FrobMetric
    shows  error rates comparable with state-of-the-art methods such as
    LMNN, ITML, and BoostMetric.
    It also performs on par with a nonlinear SVM on these datasets. 
    
    In terms of computation time, FrobMetric
    is much faster than all
    convex optimization based learning methods (LMNN, ITML, BoostMetric)
    on most data sets.
    On high-dimensional data sets with many data points, as the theory predicts,
    FrobMetric is significantly faster than LMNN. For example,
    on MNIST, FrobMetric is almost $ 140 $ times faster.
    FrobMetric is also faster than BoostMetric, although at each iteration
    the computational complexity of BoostMetric is lower. We observe that BoostMetric
    requires significantly more iterations to converge.

    Next we use FrobMetric to learn a metric for face recognition on
    the ``Labeled Faces in the Wild'' data set \cite{LFWTech}.

\subsubsection{Unconstrained face recognition}
\label{face}

\begin{figure}[t!]
    \centering
    \fbox{
    \includegraphics[width=0.12\textwidth]
                    {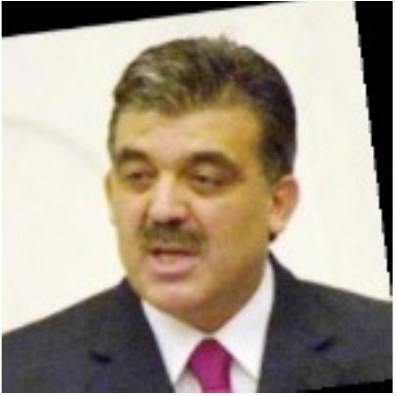}
    \includegraphics[width=0.12\textwidth]
                    {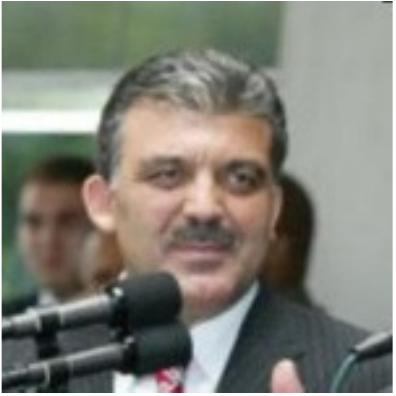}
    \includegraphics[width=0.12\textwidth]
                    {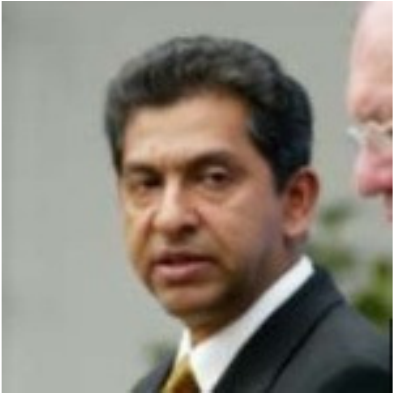}
    }\\
    \fbox{
    \includegraphics[width=0.12\textwidth]
                    {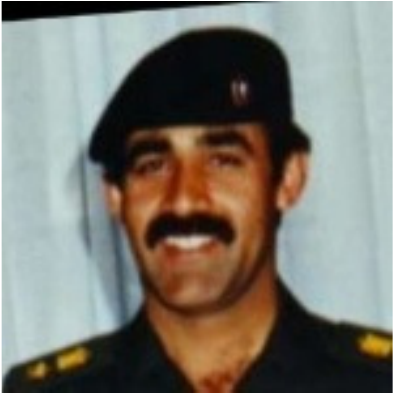}
    \includegraphics[width=0.12\textwidth]
                    {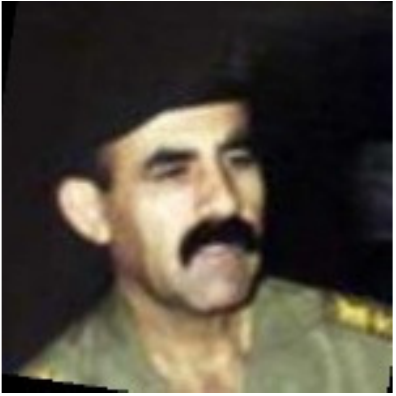}
    \includegraphics[width=0.12\textwidth]
                    {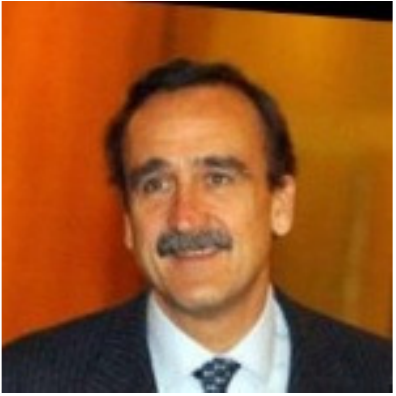}
    }\\
    \fbox{
    \includegraphics[width=0.12\textwidth]
                    {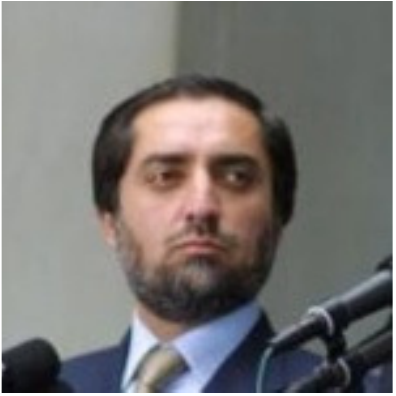}
    \includegraphics[width=0.12\textwidth]
                    {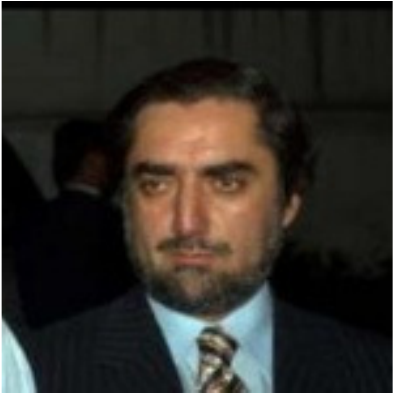}
    \includegraphics[width=0.12\textwidth]
                    {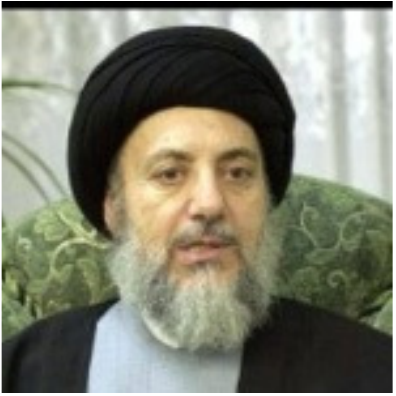}
    }\\
    \fbox{
    \includegraphics[width=0.12\textwidth]
                    {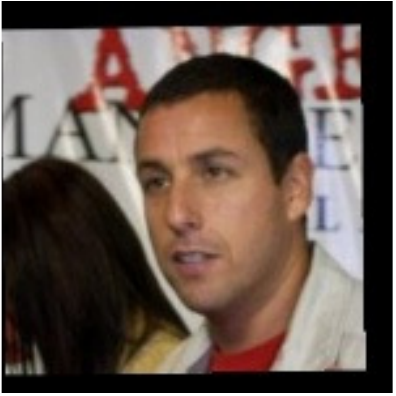}
    \includegraphics[width=0.12\textwidth]
                    {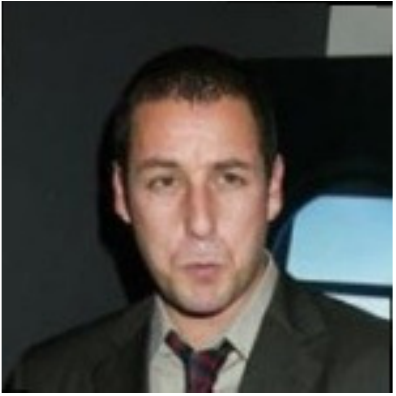}
    \includegraphics[width=0.12\textwidth]
                    {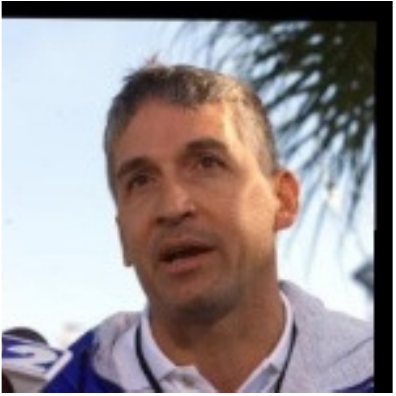}
    }
    \caption{Generated triplets based on pairwise information provided by the
    LFW data set. The first two belong to the same individual and the third is a
    different individual.}
    \label{fig:LFW_trip}
\end{figure}

    In this experiment, we have compared the proposed FrobMetric to 
    state-of-the-art methods for the task of face pair-matching problem
    on the ``Labeled Faces in the Wild'' (LFW) \cite{LFWTech} data set.
    This is a data set of
    unconstrained face images,  including $13,233$ images of $5,749$ people collected
    from news articles on the internet.  The dataset is particularly interesting because it captures much of the variation seen in real images of faces.
    The face recognition task here is to determine whether a presented pair of images are
    of the same individual.
    So we classify
    unseen pairs whether each image in the pair indicates same
    individual or not, by applying M$k$NN of \cite{Guillaumin09isthat} instead
    of $k$NN.

    Features of face images are extracted by computing $3$-scale,
    $128$-dimensional SIFT descriptors \cite{lowe2004sift}, which
    center on $9$ points of facial features extracted by a facial
    feature descriptor,  as described in \cite{Guillaumin09isthat}. PCA
    is then performed on the SIFT vectors to reduce the dimension to between $100$ and $400$.

    Since the proposed FrobMetric method
    adopts the triplet-training concept, we need to use individual's
    identity information to generate the third example in a triplet, given a pair.
    For \textit{matched} pairs, we find the third example that belongs to a \textit{different}
    individual with $k$ nearest neighbors
    ($ k $ is between $5$ and
    $30$). For \textit{mismatched} pairs, we find the $ k $ nearest neighbors ($k$ is between $5$ to
    $30$) that have the same identity as one of the individuals in the given pair.
    Some of the generated triplets are
    shown in Figure~\ref{fig:LFW_trip}.
    We select the regularization parameter using cross validation
    on View 1 and  train and test the metric using the $10$ provided
    splits in View 2 as suggested by \cite{LFWTech}.

\begin{table*}[tb!]
\caption{ {Comparison of the face recognition performance accuracy (\%) and CPU time of
our proposed FrobMetric on LFW datasets varying PCA dimensionality
and the number of triplets in each fold for training.} }
\centering {
 \renewcommand{\arraystretch}{1.25}
\begin{tabular}{r||c|c|c|c}
        \hline
\# {triplets}  & 100D & 200D & 300D& 400D\\
 \arrayrulecolor{tabgrey}
\hline \hline

\textbf{Accuracy} & & & \\ \hline

 3,000  & 82.10 (1.21) & 83.29 (1.59) & 83.81 (1.04) & 84.08 (1.18)\\
\hline

 6,000 & 82.26 (1.27) & 83.55 (1.28)& 84.06 (1.06) & 83.91 (1.48) \\ \hline

 9,000 & 82.40 (1.30) & 83.62 (1.18)& 84.08 (0.92) & 84.34 (1.23) \\ \hline

12,000 & 82.50 (1.22) & 83.86 (1.18)& 84.13 (0.84) & 84.19 (1.31) \\
\hline

15,000 & 82.55 (1.30) & 83.70 (1.22)& 84.29 (0.77) & 84.27 (0.90) \\
\hline

18,000 & 82.72 (1.24) & 83.69 (1.23) & 84.20 (0.84) & 84.32 (1.45) \\
\hline \hline

\textbf{CPU Time} & & & \\ \hline

3,000 & 51s & 215s & 373s & 937s \\ \hline

6,000 & 100s & 222s & 661s  & 1,312s \\ \hline

9,000 & 142s & 534s & 1,349s  & 3,499s \\ \hline

12,000 & 186s & 647s & 1,295s  & 6,418s \\ \hline

15,000 & 235s & 704s & 1,706s  & 3,616s \\ \hline

18,000 & 237s & 830s & 2,342s  & 7,621s \\
\arrayrulecolor{black}
\hline
\end{tabular}
} \label{Table:lfw_error2}
\end{table*}

\begin{figure}[htb!]
\begin{center}
\begin{tabular}{c}
\includegraphics[width=0.45\textwidth]{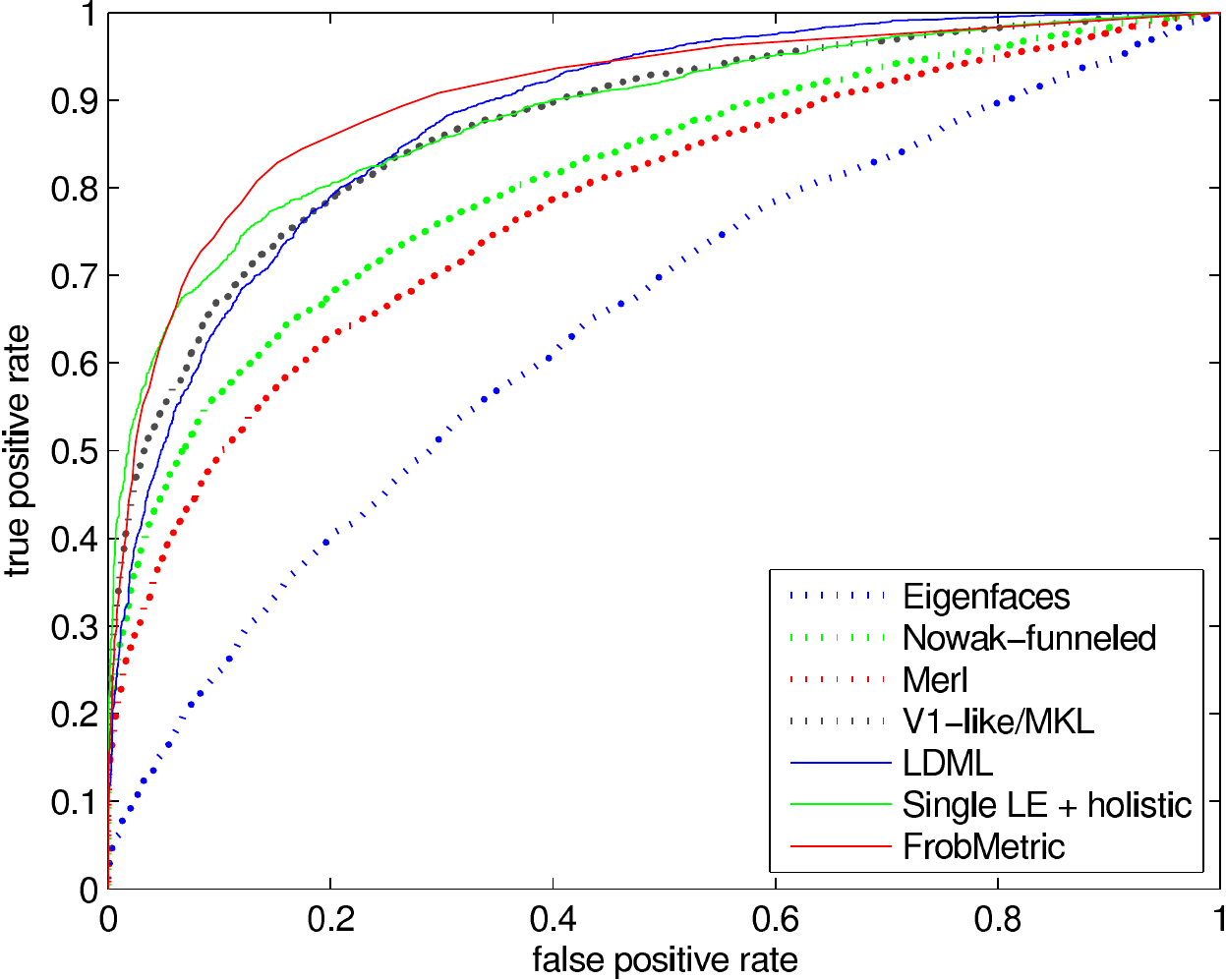}\\(a)
\\
\includegraphics[width=0.45\textwidth]{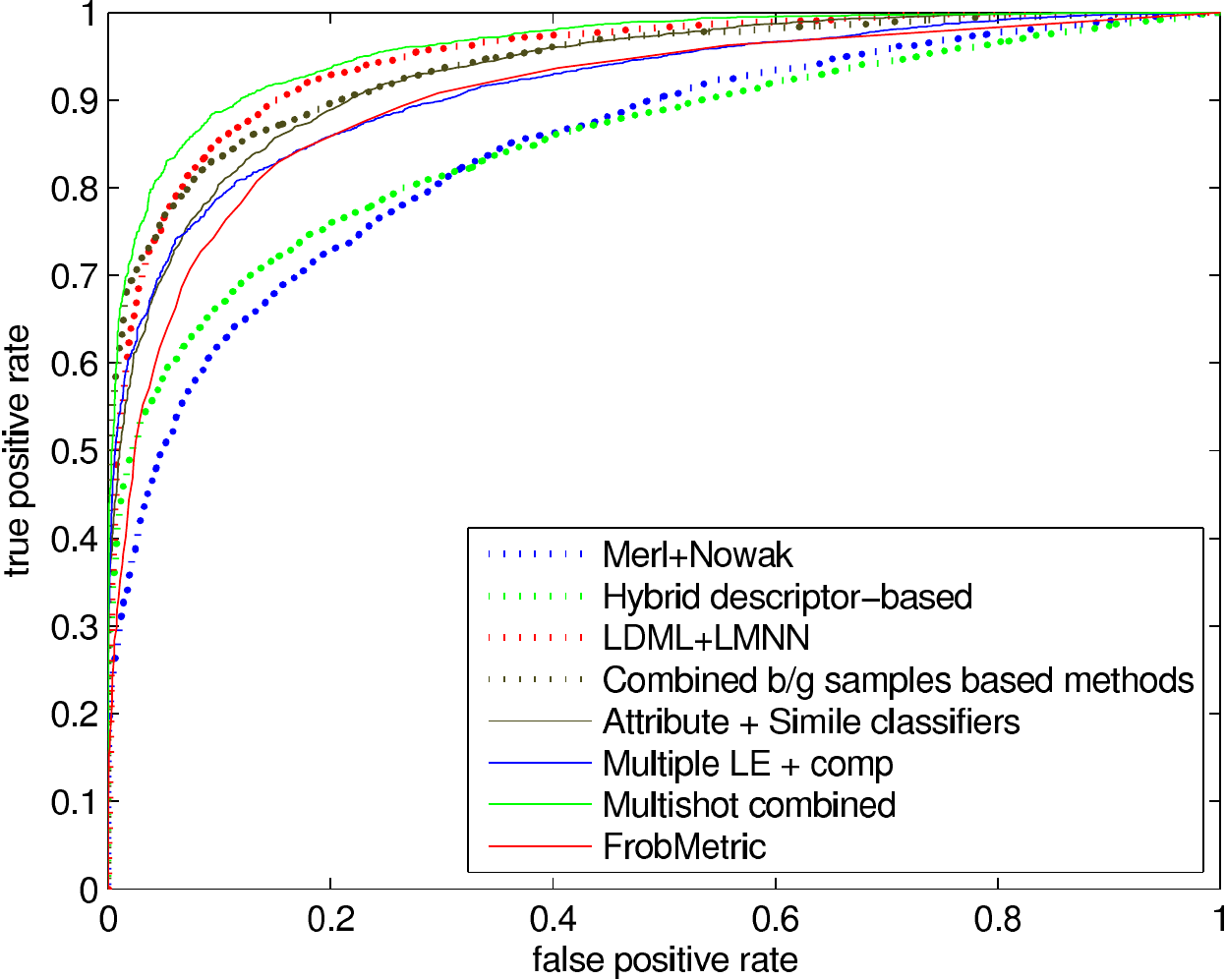}\\(b)
\end{tabular}
\end{center}
\caption{ {(a) ROC Curves that use a single descriptor and
a single classifier. (b) ROC curves that use hybrid descriptors or
single classifiers and FrobMetric's curve. Each point on a curve
is the average over the 10 runs.}} \label{fig:LFW_ROC}
\end{figure}

\textbf{Simple recognition systems with a single descriptor}
    Table~\ref{Table:lfw_error2} %
    shows the performance of 
    FrobMetric's under varying PCA dimensionality and 
    number of triplets. Increasing the number of training triplets gives a
    slight improvement in recognition accuracy.  The dimension after
    PCA has more impact on the final accuracy for this task.
    We also report the CPU time required.

    In Figure~\ref{fig:LFW_ROC} we show ROC curves for FrobMetric and related face 
    recognition algorithms. These curves were generated by altering the threshold 
    value across the distributions of match and
    mismatch similarity scores within M$k$NN. Figure~\ref{fig:LFW_ROC} (a) shows
    methods that use a single descriptor and a single classifier only. As
    can be seen, our system using FrobMetric outperforms all others.

\textbf{Complex recognition systems with one or more
descriptors}
    Figure~\ref{fig:LFW_ROC} (b) plots the performance of more complicated
    recognition systems that use hybrid descriptors or combinations of
    classifiers. See Table \ref{Table:lfw_error} for details.
    \comment{
    Cao~\etal~\cite{Cao10Face} performed a pose-adaptive matching
    applying pose-specific classifiers on LFW, labeled `Single LE +
    holistic' and `Multiple LE + comp' in Figure~\ref{fig:LFW_ROC}.
    Kumar~\etal~\cite{Kumar09Attribute} show `Attribute + Simile
    classifiers' used additional features such as pose and expression.
    Varying pose and facial expression results in larger image
    differences. Therefore it is obvious that knowing such additional
    information would boost the recognition performance. In addition,
    Kumar~\etal collected huge sized labels outside of the LFW dataset.
    }

\begin{table*}[tbh!]
\caption{\text{Test accuracy (\%) on
LFW datasets. ROC curve labels in Figure \ref{fig:LFW_ROC} are described here with details.} }
\centering {
     \renewcommand{\arraystretch}{1.25}
\begin{tabular}{r||l|l}
        \hline
  & SIFT or single descriptor + single classifier & multiple descriptors or classifiers\\
   \arrayrulecolor{tabgrey}
\hline \hline

Turk~\etal~\cite{Matthew91Face} & 60.02 (0.79) & - \\
 & {`Eigenfaces'} &  \\
\hline

Nowak~\etal~\cite{Nowak07Learning} & 73.93 (0.49) & - \\
 & {`Nowak-funneled'} &  \\
\hline

Huang~\etal~\cite{Huang08LFW} & 70.52 (0.60) &  76.18 (0.58)\\
 & {`Merl'} & {`Merl+Nowak'} \\
\hline

Wolf~\etal in 2008\cite{Wolf08Descriptor} & - &  78.47 (0.51)\\
 &  & {`Hybrid descriptor-based'} \\
\hline

Wolf~\etal in 2009\cite{Wolf09Similarity} & 72.02 & 86.83 (0.34) \\
& - & {`Combined b/g samples based methods'}\\
\hline

Pinto~\etal~\cite{Pinto09How} & 79.35 (0.55) &  - \\
 & {`V1-like/MKL'} &  \\
\hline

Taigman~\etal~\cite{Taigman09Multiple} & 83.20 (0.77) & \textbf{89.50 (0.40)}\\
 & - & {`Multishot combined'}\\
\hline

Kumar~\etal~\cite{Kumar09Attribute} & - & 85.29 (1.23)\\
 & & {`attribute + simile classifiers'}\\
\hline

Cao~\etal~\cite{Cao10Face} & 81.22 (0.53) & 84.45 (0.46) \\
 & {`single LE + holistic'} & {`multiple LE + comp'} \\
\hline

Guillaumin~\etal~\cite{Guillaumin09isthat} & 83.2 (0.4) &  87.5 (0.4)\\
 & {`LDML'} & {`LMNN + LDML'}\\
\hline

FrobMetric (this work) & \textbf{84.34 (1.23)} & - \\
 & {`FrobMetric' on SIFT } &  \\
  \arrayrulecolor{black}
\hline

\end{tabular}
} \label{Table:lfw_error}
\end{table*}

\comment{
    Taigman~\etal~\cite{Taigman09Multiple}'s work separated facial
    images using pose information to improve performance as well.
    Moreover, their best performance, `Multishot combined' used 8
    descriptors (SIFT, LBP, TPLBP and FPLBP, and four with square root
    of 4 descriptors) and combined 16 scores using the ITML + Multiple
    OSS ID method and the pose-based multiple shots. Applying SVMs
    classifier, the accuracy of `Multishot combined', 89.50\% became the
    best among computer vision approaches on LFW in the literature.
    However, when their system used Multiple OSS with SIFT only and did
    not apply pose information, the accuracy was decreased to 83.20\%
    while our FrobMetric's best accuracy was 84.34\% as shown in
    Table~\ref{Table:lfw_error}.

    Wolf~\etal~\cite{Wolf09Similarity} used a hybrid descriptor which
    has 10 distances, 10 One-Shot distances, 10 Two-Shot distances, 10
    ranking based distances and 20 additional dimensions using LDA.
    Applying the hybrid descriptors and adding SVM, the paper achieved
    86.83\% accuracy which labeled `Combined b/g samples based methods',
    whereas the accuracy became 72.02\% when single SIFT feature in
    Funneled image was used at this task.

    Guillaumin~\etal~\cite{Guillaumin09isthat}, `LDML + LMNN' in
    Figure~\ref{fig:LFW_ROC} combined 8 scores of 4 descriptors with the
    combination of LDML and LMNN. Their best results was 87.5\% while
    LDML using single descriptor showed 83.20\% accuracy.
    }

    As stated above, the leading algorithms have used either 1) additional
    appearance information, 2) multiple scores from multiple descriptors,
    or 3) complex recognition systems with hybrids of two or more methods.
    In contrast, our system using FrobMetric employs neither a combination of
    other methods nor  multiple descriptors. That is, our
    system exploits a very simple recognition pipeline. 
    The method thus reduces the computational costs associated with 
    extracting the
    descriptors, generating the prior information, training, and
    computing the recognition scores.

    With such a simple metric learning approach, and modest computational cost, 
    it is notable that the method is only slightly outperformed by state-of-the-art hybrid systems
    (test accuracy of $84.34\% \pm 1.23\%$ versus
    $89.50\% \pm 0.40\% $
    on the
    LFW datasets).
    We would expect that the accuracy of the
    FrobMetric approach would improve similarly if more features,  such as local binary pattern (LBP)
    \cite{TPAMILBP} for instance, were used.

		The FrobMetric approach shows better classification performance at a lower computational cost 
		than comparable single descriptor methods.  Despite this level of performance it is 
		surprisingly simple to implement, in comparison to the state-of-the-art.

\subsubsection{Metric learning for action recognition}

In this experiment, we compare the performance of the proposed method with that of existing approaches on two action recognition benchmark data sets, KTH \cite{localsvm} and Weizmann \cite{pami07}. Some examples of the actions are shown
in Figure~\ref{fig:action-examples}. We aim to demonstrate again the advantage of our method in
reducing computational overhead while achieving excellent recognition performance.

The KTH dataset in this experiment consists $2,387$ video sequences. They can be categorized into
six types of human actions including \textit{boxing, hand-clapping, jogging, running, walking and
hand-waving}. These actions are conducted by $25$ subjects and each action is performed multiple
times by a same subject. The length of each video is about four seconds at $25$ fps, and the
resolution of each frame is $160\times{120}$. 
We randomly split all the video sequences based on
the subjects into $10$ pairs, each of which contains all the sequences from $16$ subjects for
training and those from the remaining $9$ subjects for test. 
The space-time interest points (STIP)
\cite{Laptev08} were extracted from each video sequence and the
the corresponding descriptors calculated.
The descriptors extracted from all training sequences were clustered into $4,000$
clusters using $k$-means, with the cluster centers used to form a visual codebook.
Accordingly, each video sequence is characterized by a $4000$-dimensional histogram indicating
the occurrence of each visual word in this sequence. To achieve a compact and discriminative
representation, a recently proposed visual word merging algorithm, called AIB
\cite{FulkersonECCV08}, is
applied to merge the histogram bins to reduce the dimensionality. Subsequently each video sequence is
represented by a $500$-dimensional histogram.

The Weizmann data set contains temporal segmentations of video sequences into ten types of human actions
including \textit{running, walking, skipping, jumping-jack, jumping-forward-on-two-legs,
jumping-in-place-on-two-legs, galloping-sideways, waving-two-hands, waving-one-hand, and bending}.
The actions are performed by $9$ actors. The action video sequences are represented by
space-time shape features such as space-time ``saliency'', degree of ``plateness'', and degree of
``stickness'' which compute the degree of location and orientation movement in space-time domain by
a Poisson equation \cite{Tran08}. This leads to a $286$-dimensional feature
vector for each action video sequence, which is as  in
\cite{Tran08}. In this experiment, 
$70\%$ sequences are used for training
and the remaining $30\%$ for testing.

The experimental results are shown in Table~\ref{Table:exp-action-recognition}. The first part of
the table shows the experimental setting and the second  compares the results of various
metric learning methods. On the KTH data set, the proposed method, FrobMetric,
performs almost as well as BoostMetric with an error rate of $7.03\pm{1.46}\%$,
and outperforms all others.
In doing so FrobMetric requires only $289.58$ seconds to
complete the metric learning, which is approximately one quarter of the time required by the 
fastest competing method (which has more than double the error rate).
On the Weizmann data set, the error rate of FrobMetric is
$0.59\pm{0.20}\%$, which is the second-best among all the compared methods. It is slightly higher
than (but still comparable to) the lowest one $0.30\pm{0.09}\%$ obtained by LMNN. However, in
terms of computational efficiency, FrobMetric requres approximately one eighth of the time used by LMNN, 
and is the fastest of the methods compared. These results demonstrate the
computational efficiency and the excellent classification performance of the proposed method in
action recognition.

\begin{figure}[!t]
\centering
    \fbox{
    \begin{tabular}{ccc}
\includegraphics[width=0.13\textwidth]{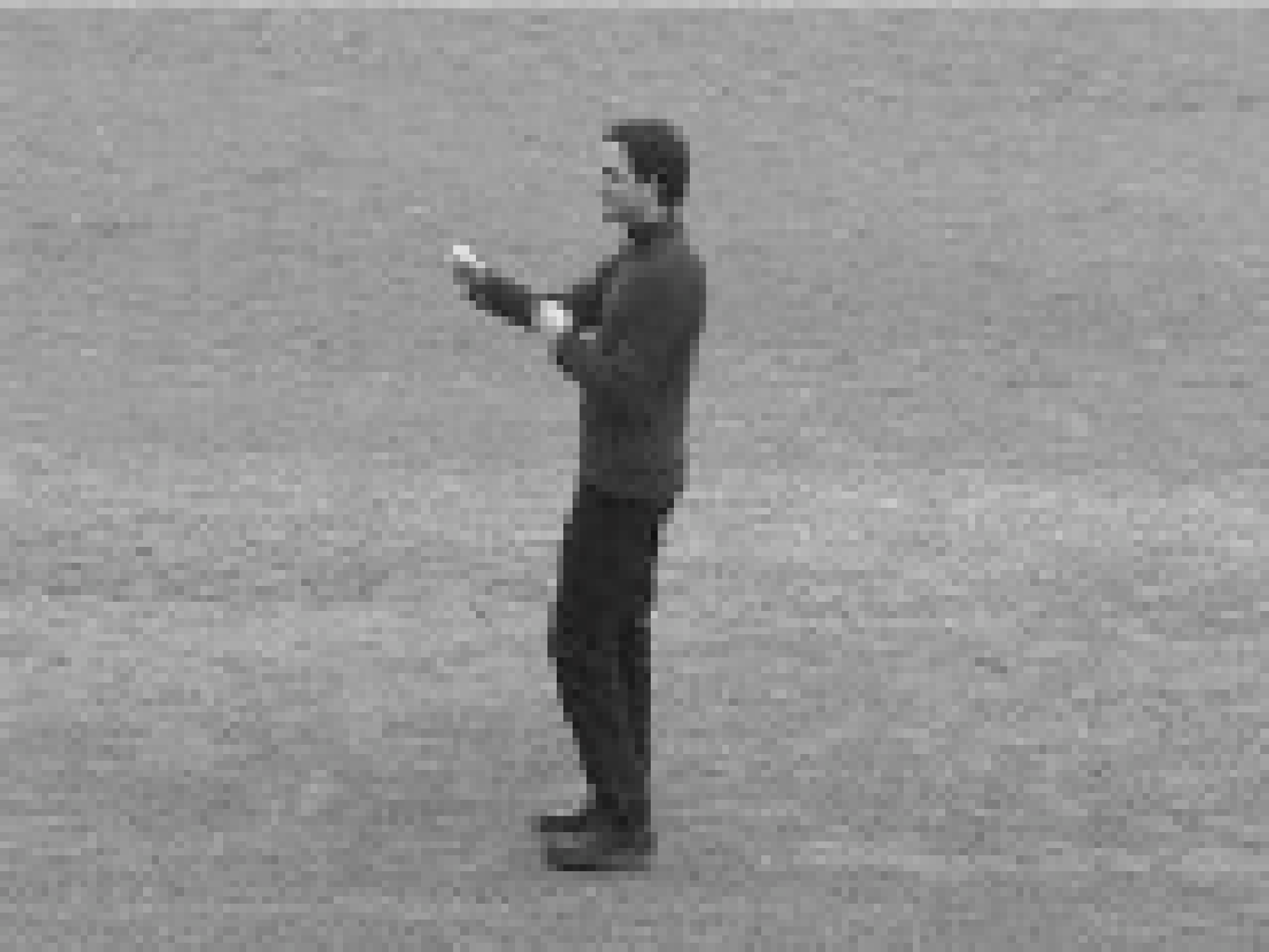} &
\includegraphics[width=0.13\textwidth]{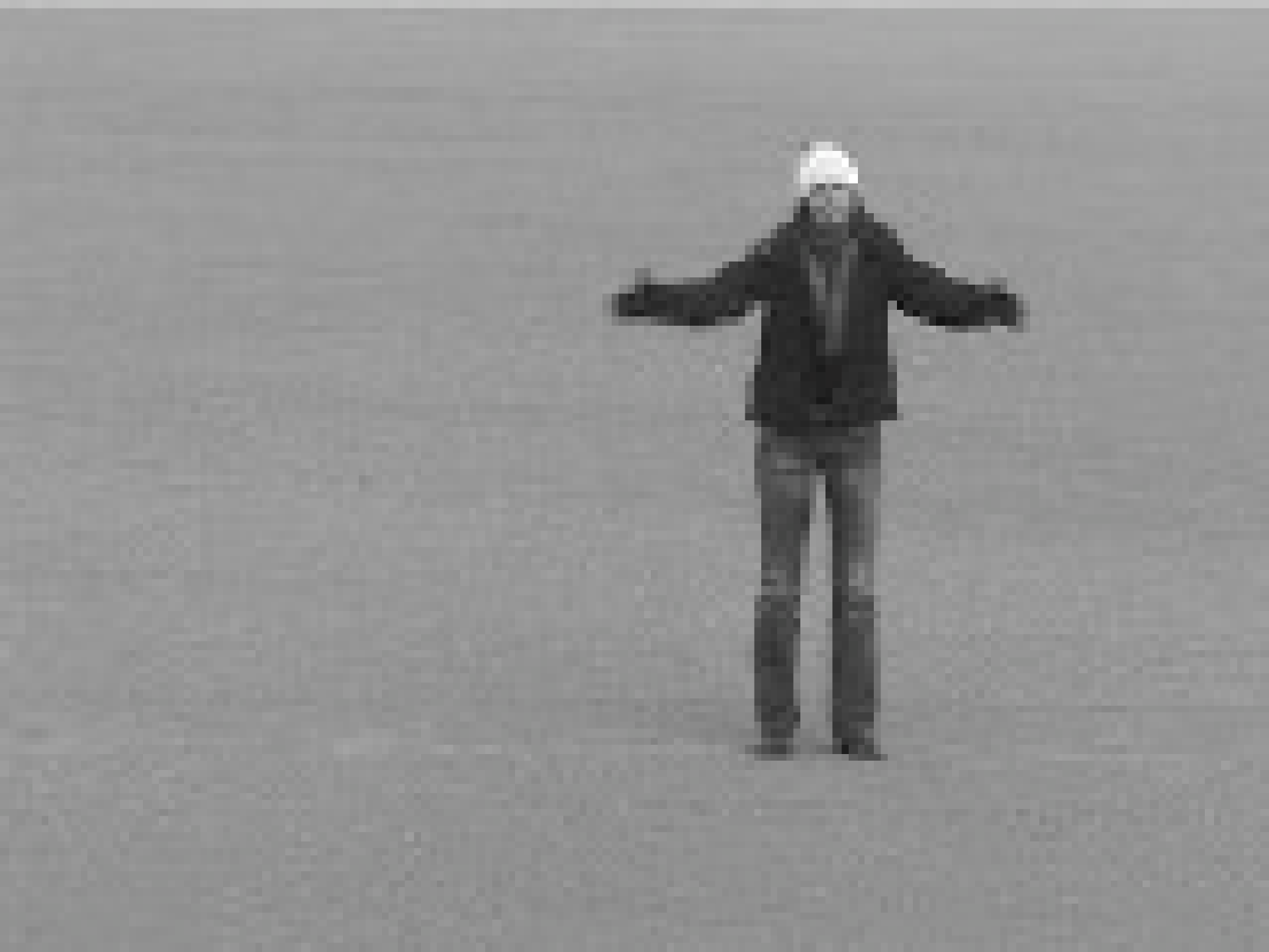} &
\includegraphics[width=0.13\textwidth]{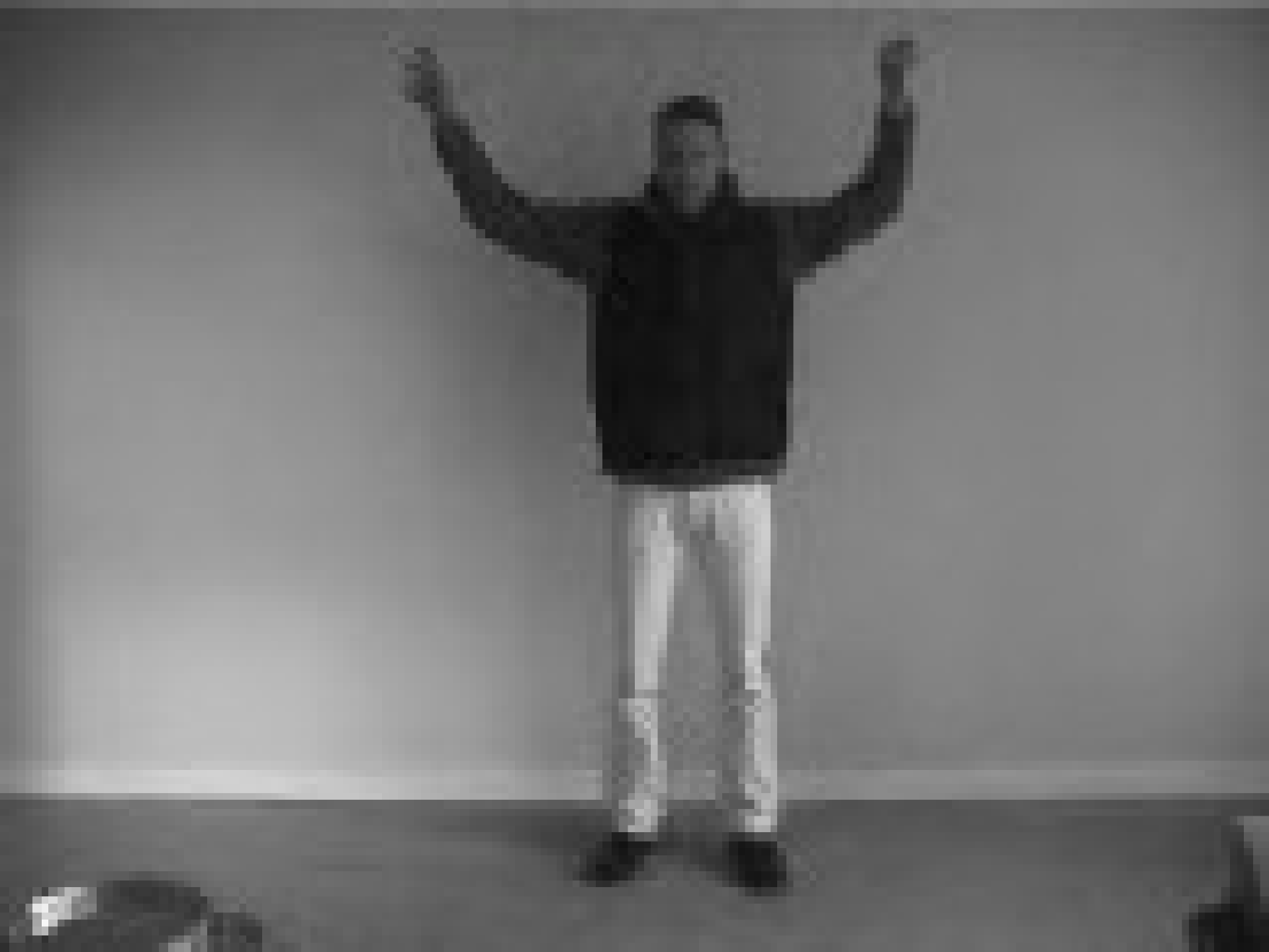}  \\
\scriptsize boxing & \scriptsize handclapping & \scriptsize handwaving \\
\includegraphics[width=0.13\textwidth]{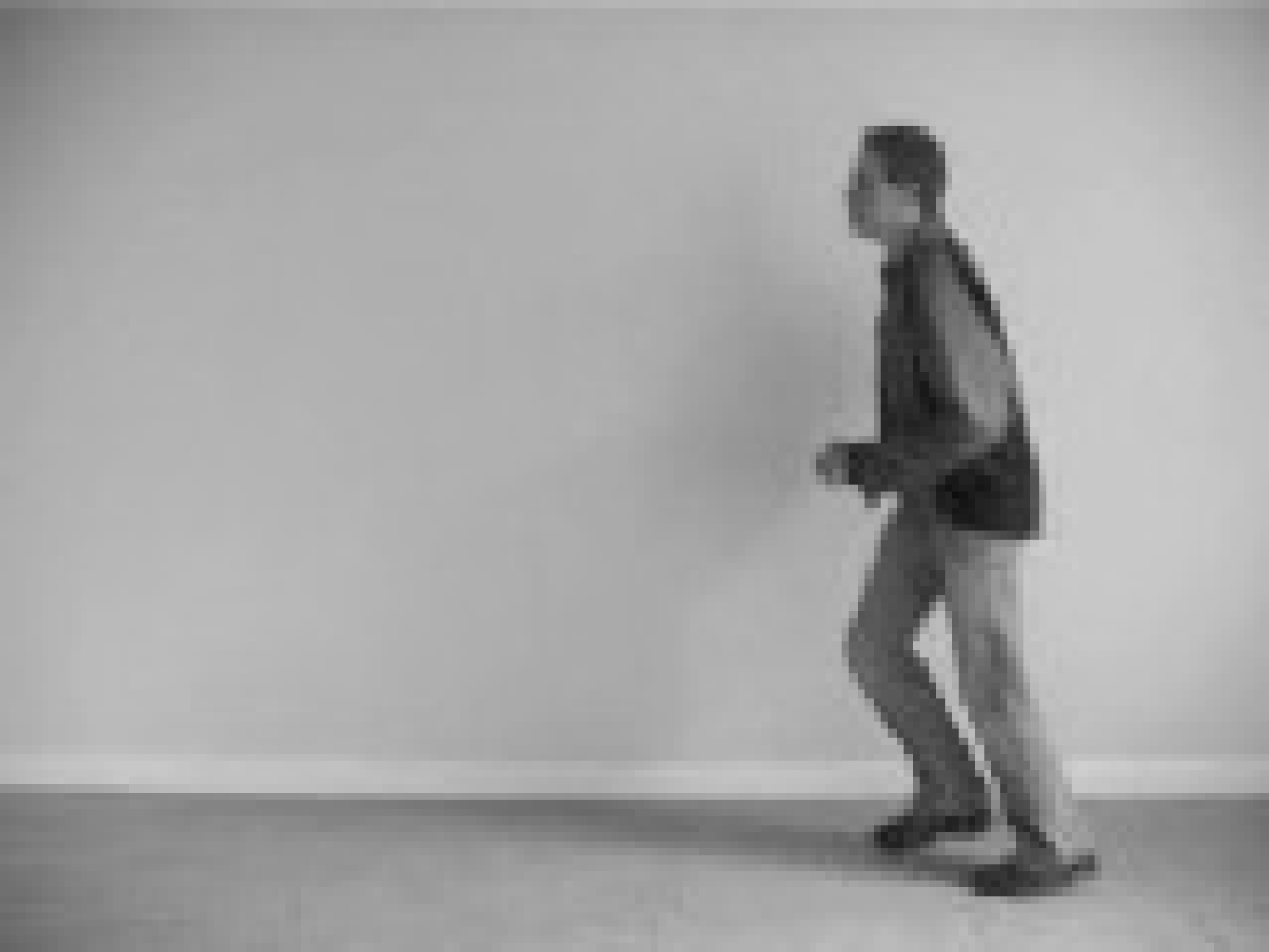} &
\includegraphics[width=0.13\textwidth]{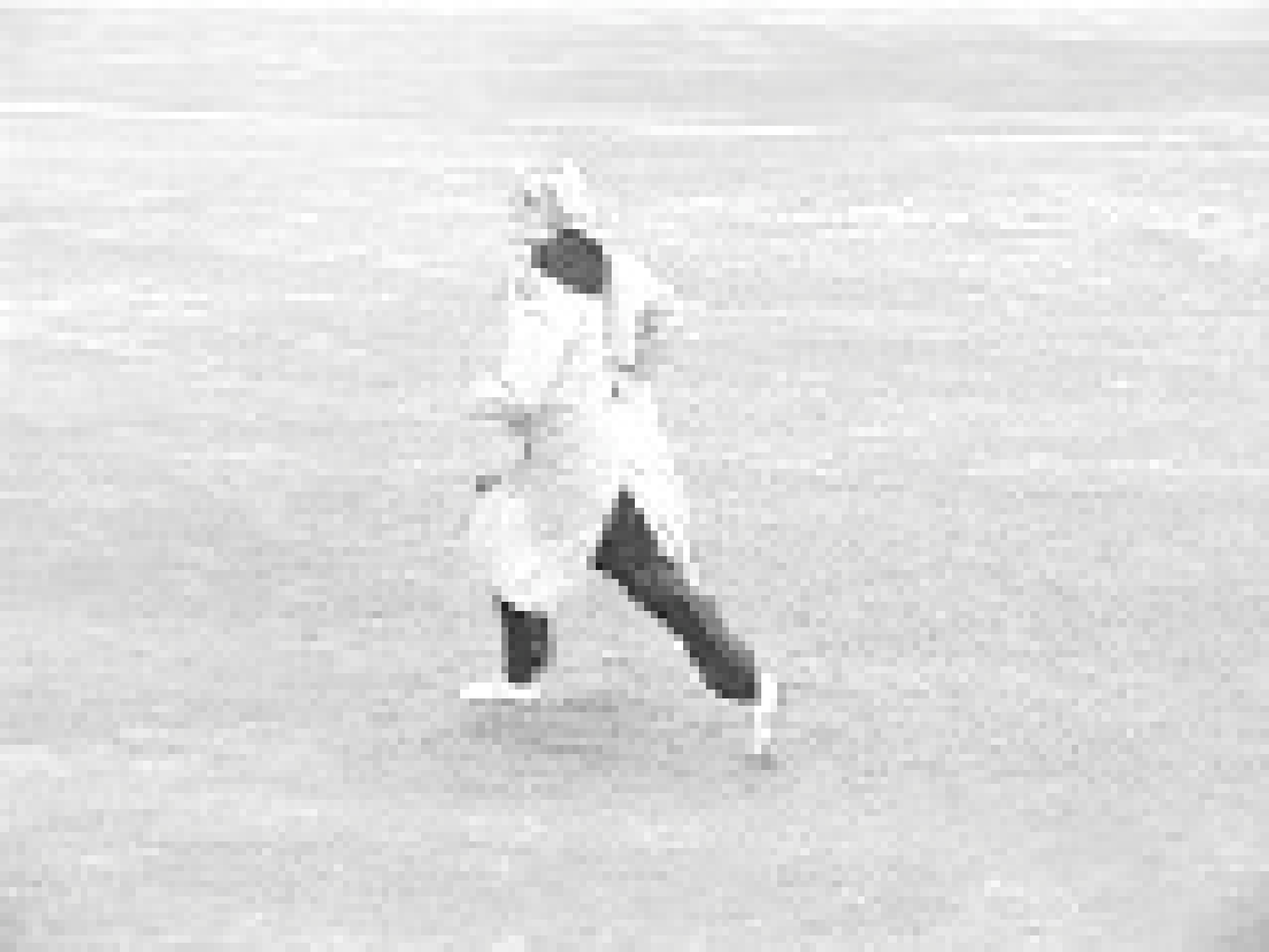} &
\includegraphics[width=0.13\textwidth]{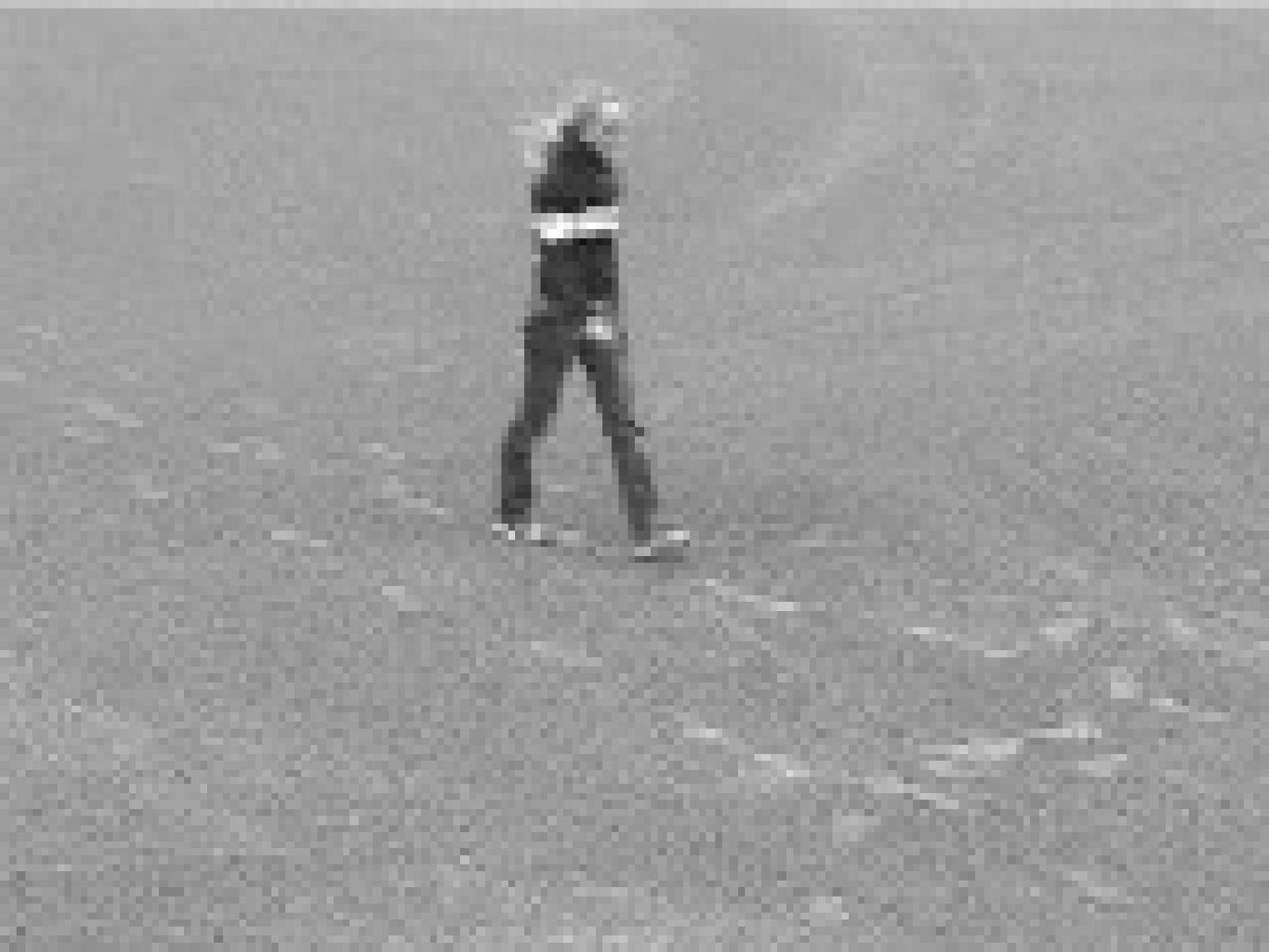} \\
 \scriptsize jogging & \scriptsize running & \scriptsize walking
\end{tabular}
}
\caption{Examples of the actions from the KTH action dataset \cite{localsvm}.}
\label{fig:action-examples}
\end{figure}

\begin{table}
\begin{center}
\caption{Comparison of FrobMetric and other metric learning methods on action recognition datasets with $3$-NN (Standard deviation is
reported for the datasets having multiple runs). }

\label{Table:exp-action-recognition}

\centering \footnotesize
\begin{tabular}{l|l||c|c}
\hline
     \multicolumn{2}{l||}{}  & KTH & Weizmann\\
\hline\hline

\multicolumn{2}{l||}{\# samples}   & 2,387 &  5,594\\\hline

\multicolumn{2}{l||}{\# triplets}  & 13,761 & 35,280 \\\hline

\multicolumn{2}{l||}{dimension}  & 500    &  286 \\\hline

\multicolumn{2}{l||}{\# training} & 1,529 & 3,920 \\\hline

\multicolumn{2}{l||}{ test } & 858  & 1,674 \\\hline

\multicolumn{2}{l||}{\# classes} & 6     & 10 \\\hline

\multicolumn{2}{l||}{\# runs}  & 10      &  10 \\

\hline \hline

\multirow{6}{*}{\textbf{Error Rates $ \%$}} &
Euclidean & 10.55 (2.46) &  1.14 (0.19)\\
\cline{2-4}

&RCA & 21.05 (3.86) &  3.21 (0.66)\\
\cline{2-4}

&LMNN  & 15.72 (2.57) & 0.30 (0.09)  \\
\cline{2-4}

&ITML  &   27.67 (1.47) & 1.06 (0.16)
\\ \cline{2-4}

&BoostMetric  & 7.05 (1.42) & 0.85 (0.31) \\
\cline{2-4}

&FrobMetric & 7.03 (1.46) & 0.59 (0.20) \\

\hline \hline

\multirow{4}{*}{\textbf{Comp. Time}}  &

LMNN & 1023.89s & 1343.25s \\
\cline{2-4}

&ITML &  1004.94s & 368.68s \\
\cline{2-4}

&BoostMetric & 4048.67 & 1139.02s \\
\cline{2-4}

&FrobMetric & 289.58s & 169.30s \\
\hline

\end{tabular}
\end{center}
\end{table}

{
\subsection{Maximum variance unfolding}

\begin{figure}[htb!]
    \centering
    \begin{tabular}{c}
    \includegraphics[width=0.3\textwidth]{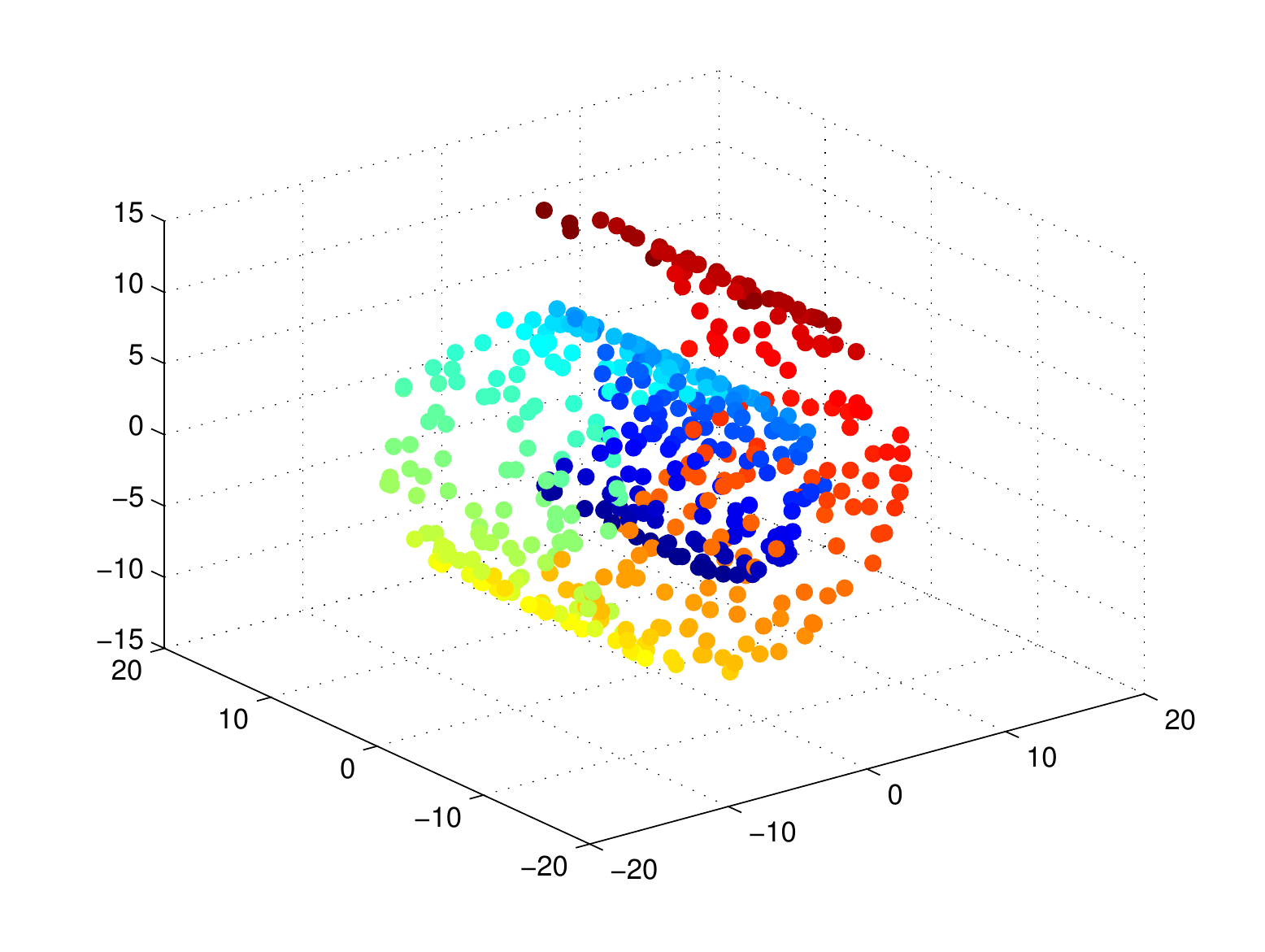} \\
    original data
    \end{tabular}
    \begin{tabular}{cc}
    \includegraphics[width=0.22\textwidth]{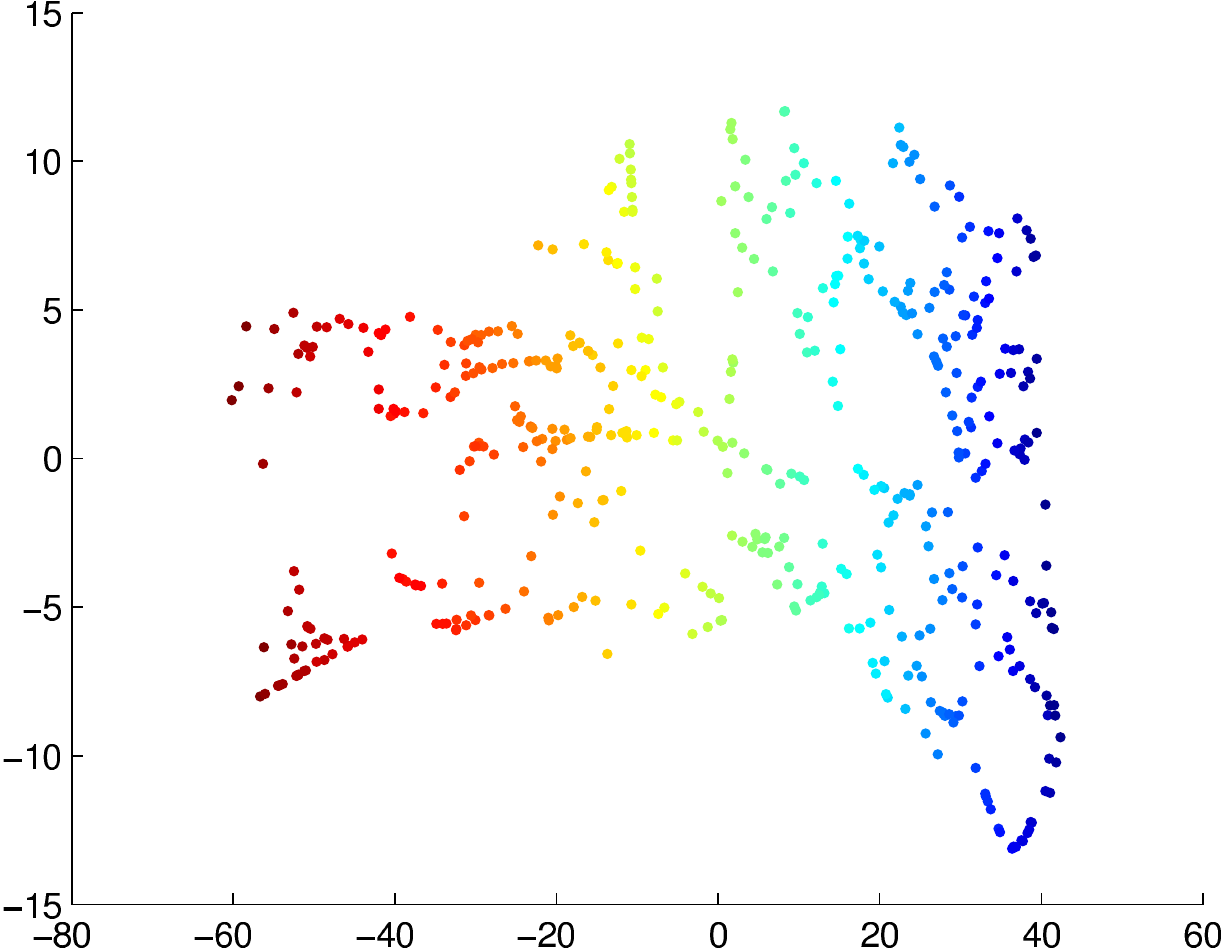} &
    \includegraphics[width=0.22\textwidth]{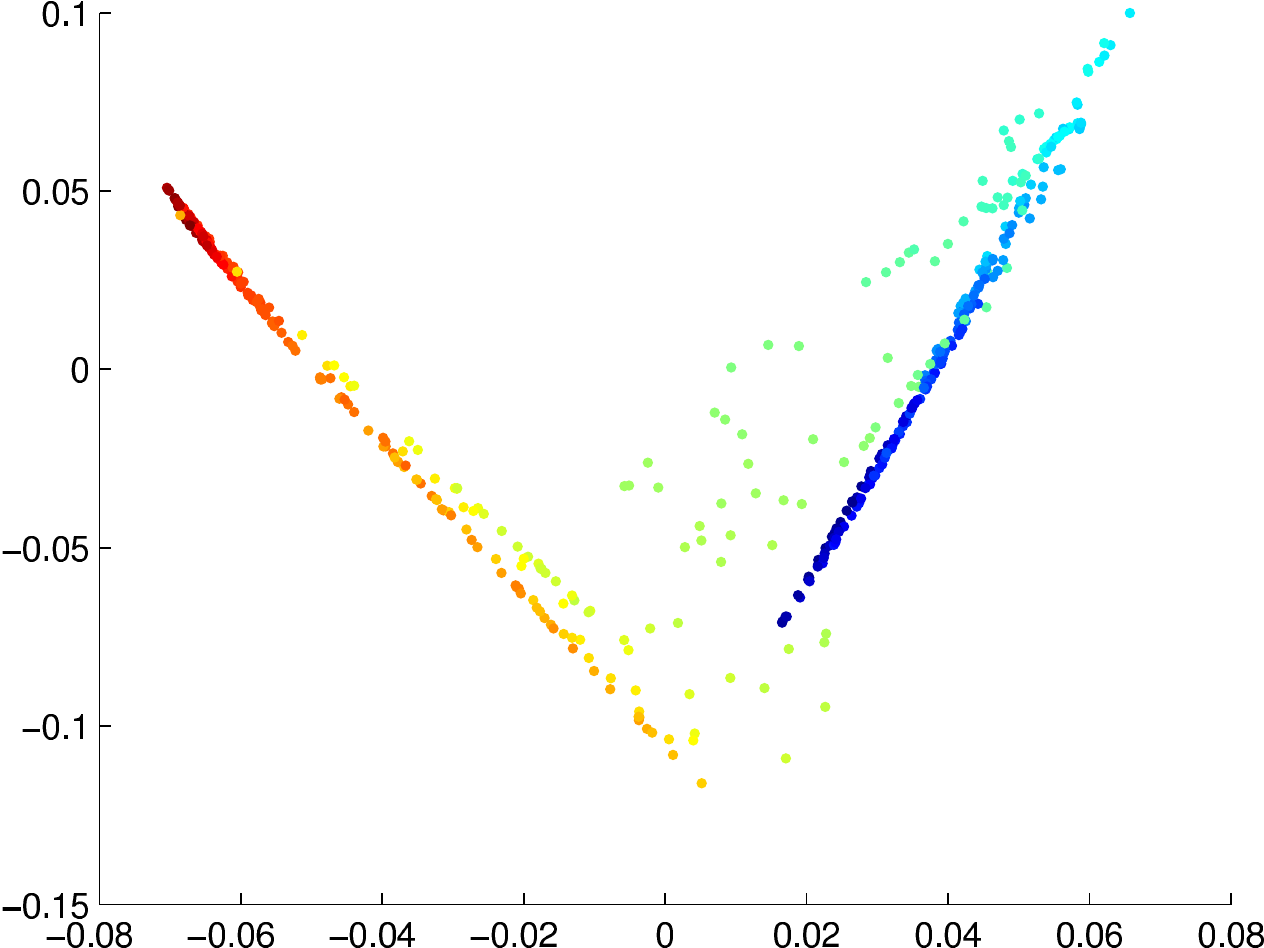}
    \\ (a) & (b)
    \end{tabular}
    \begin{tabular}{cc}
    \includegraphics[width=0.22\textwidth]{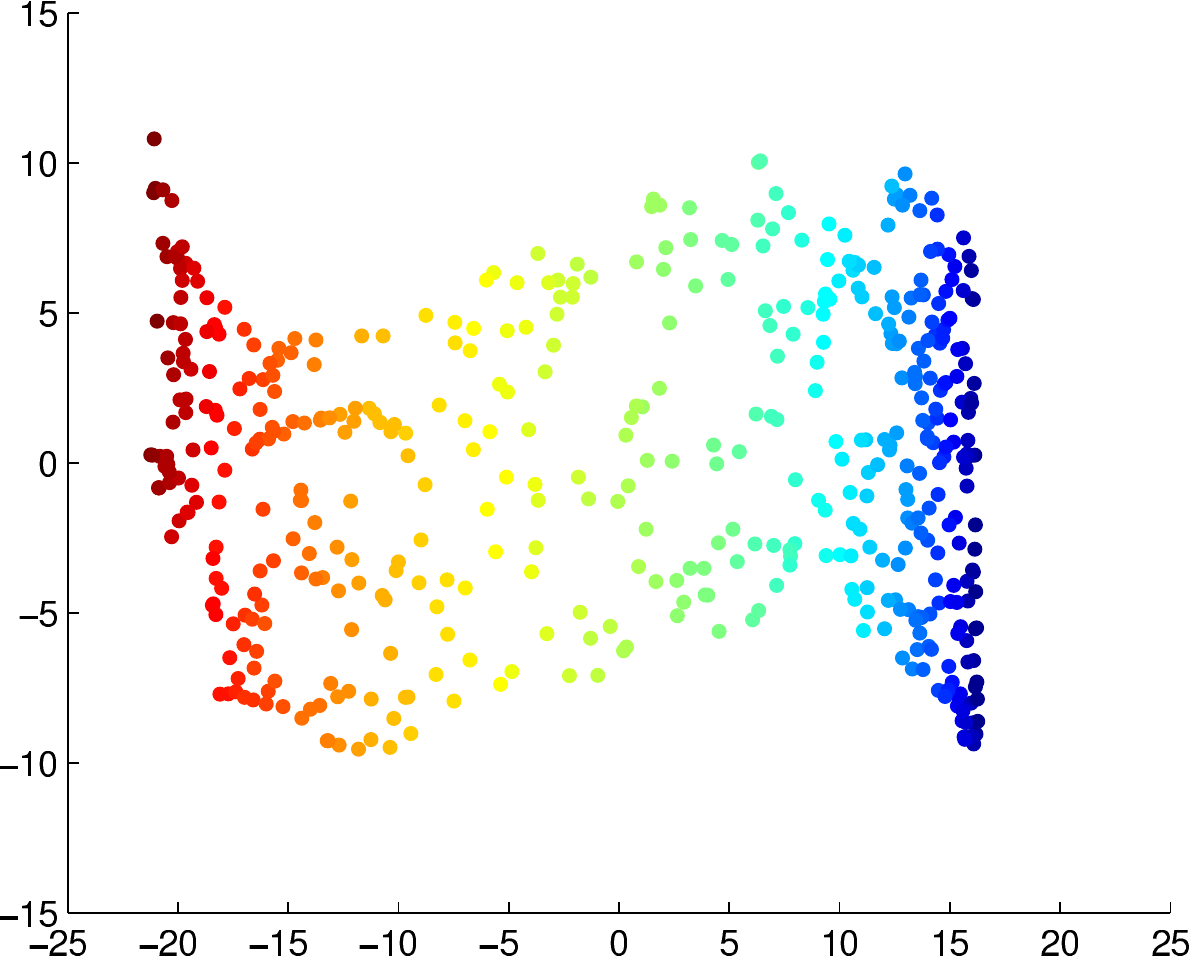} &
    \includegraphics[width=0.22\textwidth]{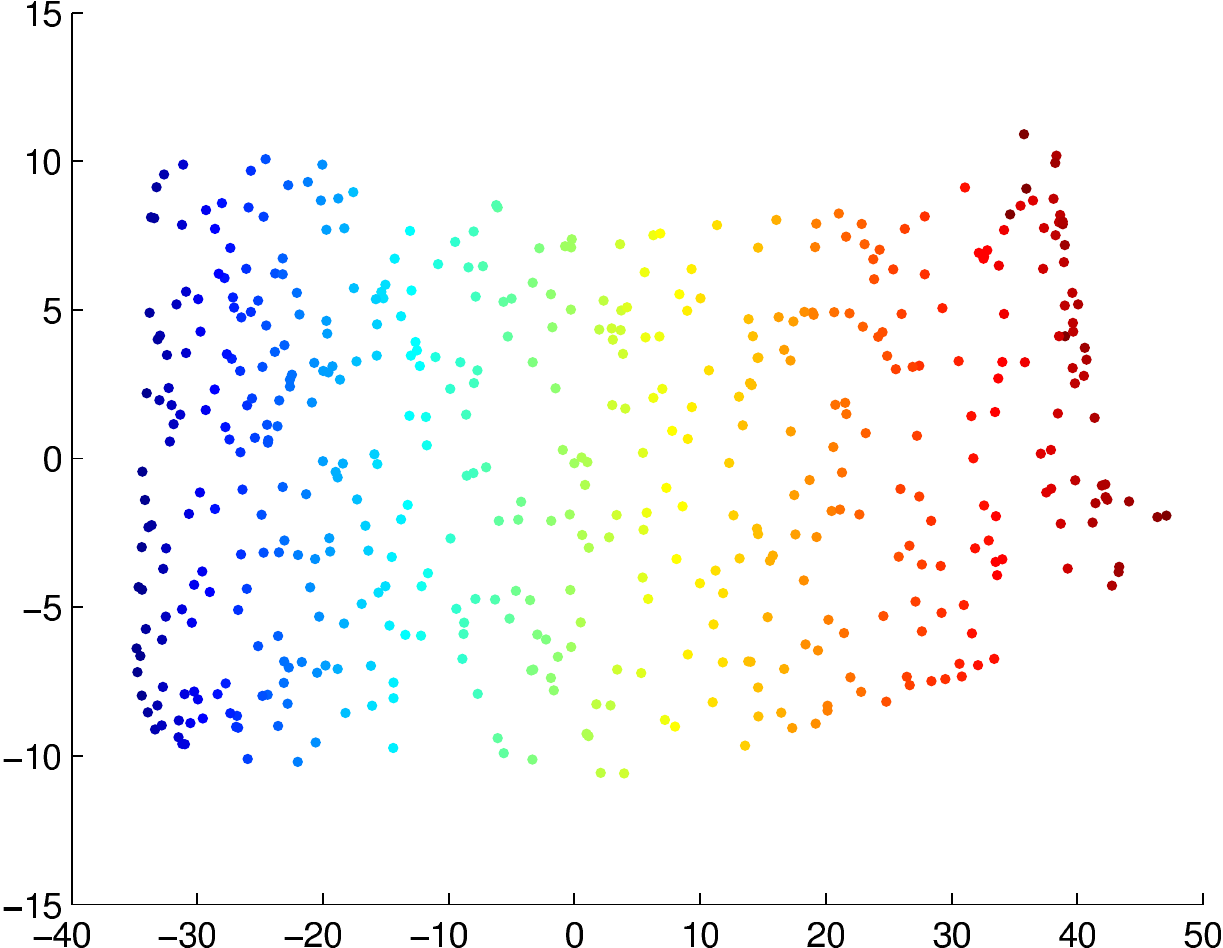}
    \\ (c) & (d)
    \end{tabular}
    \caption{Embedding results of different methods on the 3D
    swiss-roll dataset, with the neighborhood size $k=6$ for all,
    and $\sigma = 10^{5}$ for our method.
    (a) Isomap, (b) LLE, (c) Our method, (d) MVU.}
    \label{fig:swiss500}
\end{figure}

In this section, we run MVU experiments on a few datasets and compare
with other embedding methods. 
Figure~\ref{fig:swiss500} shows the embedding results for several
different methods,
namely, isometric mapping (Isomap) \cite{isomap},
locally linear embedding (LLE) \cite{lle}
and MVU \cite{mvu}
on the 3D swiss-roll with 500 points.
We use $ k = 6 $ nearest neighbors to construct the local distance constraints
and  set $ \sigma = 10^5 $.
 We have also applied our method to  the teapot and face image
 datasets from~\cite{mvu}.
The teapot set contains $200$ images obtained by rotating a teapot through $360^\circ$. Each image is
of $101 \times 76$ pixels. Figure~\ref{fig:teapot} shows the two dimensional
embedding results of our method and MVU.
As can be seen, both methods preserve the order of teapot
images corresponding to the angles from which the images were taken,
and produce
plausible embeddings. But in terms of running time, our algorithm is more than an order of magnitude faster than
MVU, requiring only $4$ seconds to run using $k = 6$ and
$\sigma = 10^{10}$, while MVU required $85$ seconds.

\begin{figure}[t]
\centering
\includegraphics[width=0.4\textwidth]{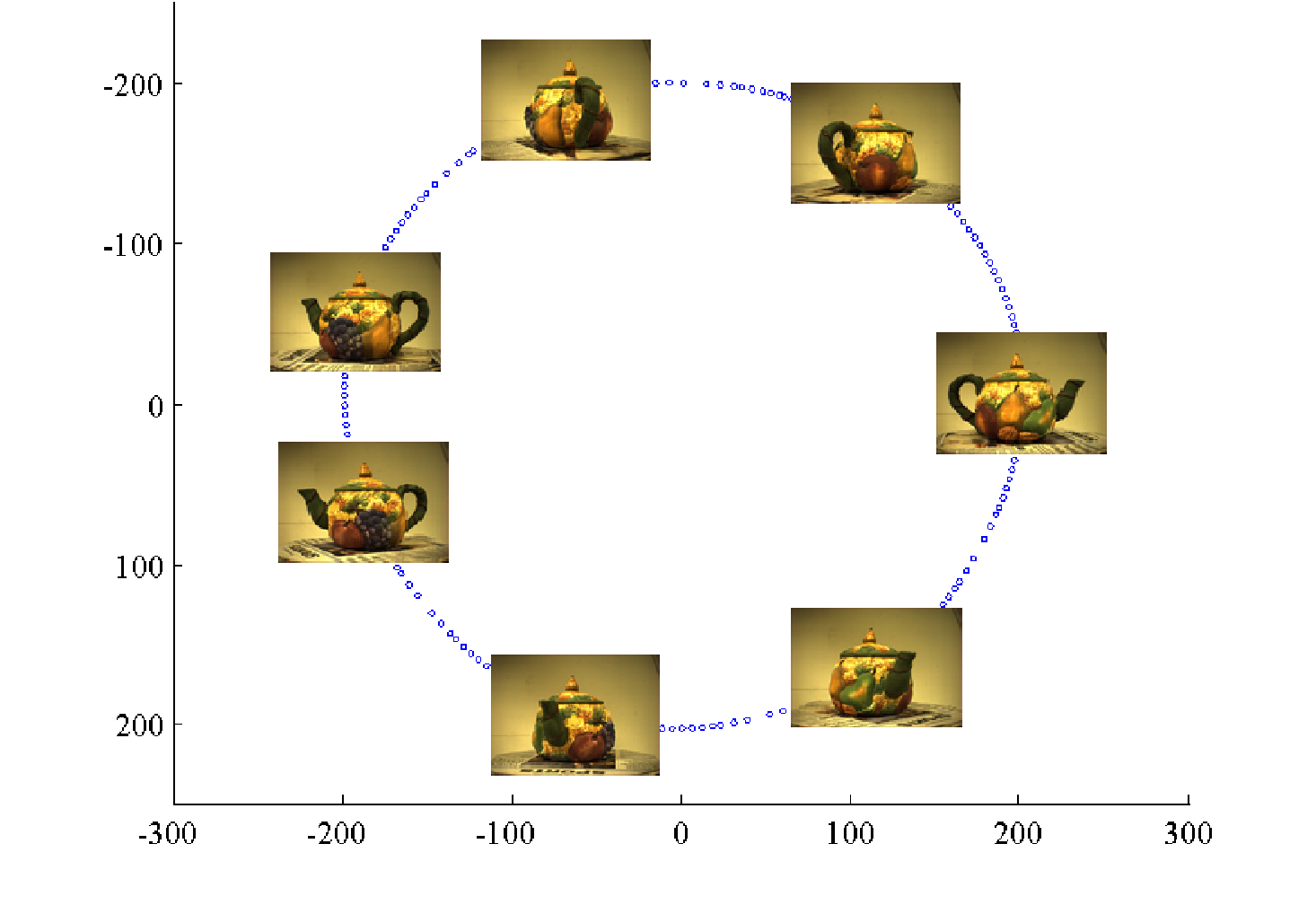}
\includegraphics[width=0.4\textwidth]{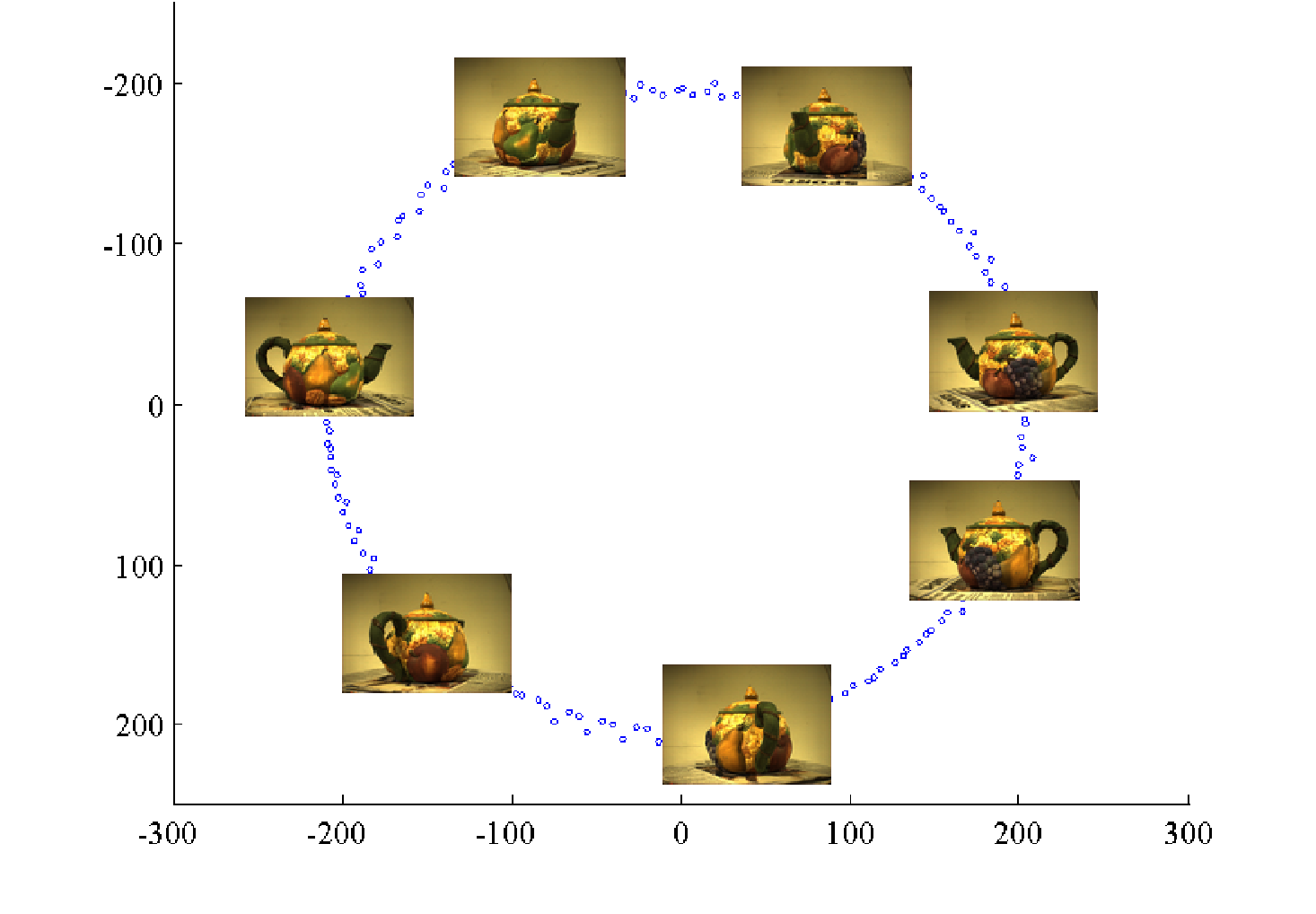}
\caption[Embedding results of our method and MVU on the teapot
dataset]{Embedding results of our method and MVU on the teapot dataset;
 (top) our results with $\sigma = 10^{10}$ and
(bottom) MVU's results.} \label{fig:teapot}
\end{figure}
Figure~\ref{fig:face_mvu} shows a two-dimensional embedding of the images from the face
dataset. The set contains $1,965$ images (at $28 \times 20$ pixels) of the same individual from different
views and with differing expressions. The proposed method required $131$ seconds
to solve this metric using $k = 5$ nearest
neighbors whereas the original MVU needed $4732$ seconds.
\begin{figure}[htb!]
\begin{center}
\includegraphics[width=0.5\textwidth]{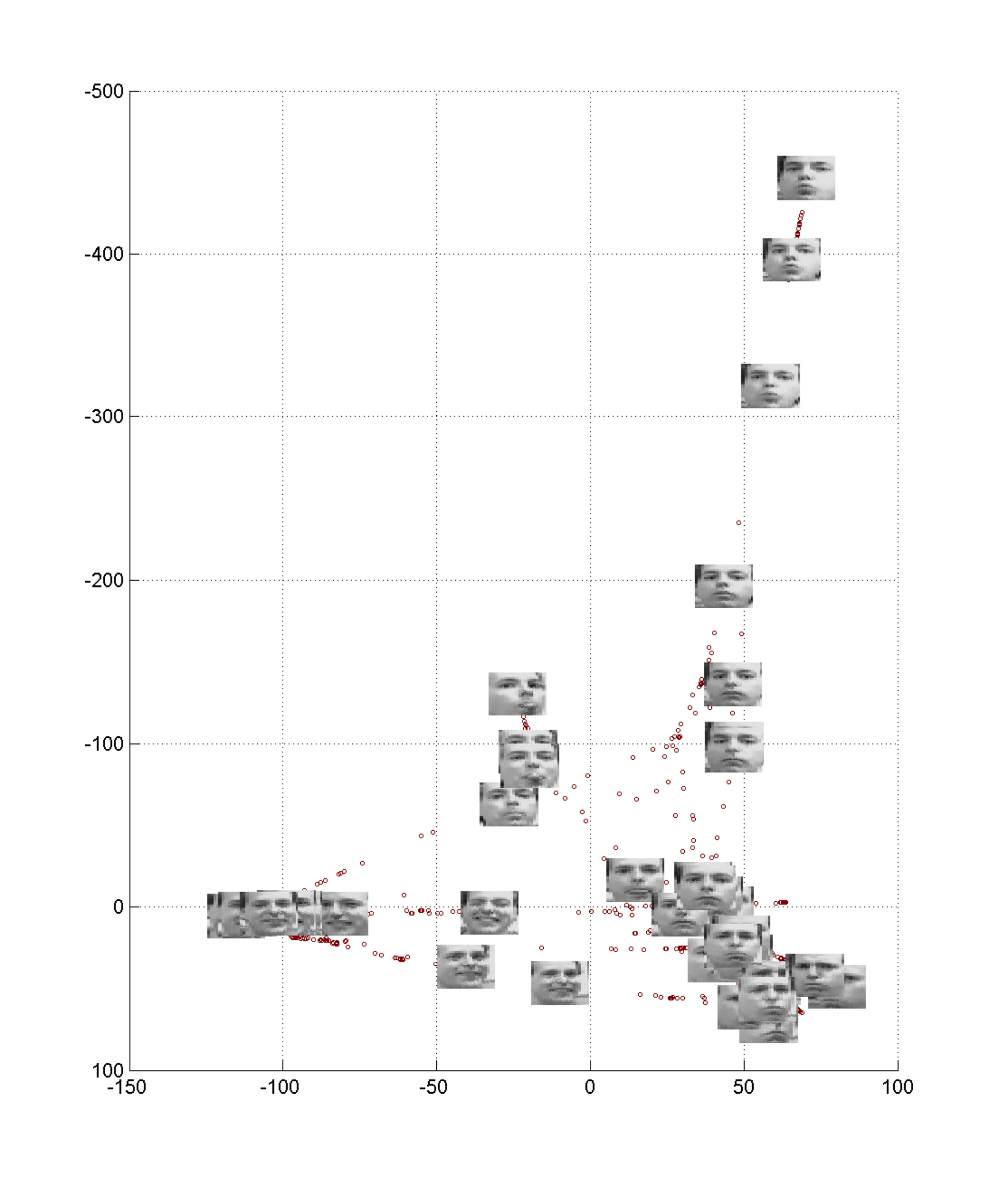}
\end{center}
\caption{2D embedding of face data by our approach.}
\label{fig:face_mvu}
\end{figure}
\subsubsection{Quantitative Assessment}

To better illustrate the effectiveness of our method, here
we provide a quantitative evaluation of the embeddings generated for the 3D swiss-roll and teapot
datasets.     
Specifically, we adopt two quality mapping
indexes, the unweighted $Q_{nx}$ and $B_{nx}$ \cite{LeeV08jmlr}, to measure the
K-ary neighborhood preservation between the high and low dimensional spaces.
$Q_{nx}$ represents the proportion of points that remain inside the
K-neighborhood after projection,
and thus larger $Q_{nx}$ indicates better neighborhood preservation.
$B_{nx}$ is defined as the difference in the fractions of mild
K-extrusions and mild K-intrusions.
It indicates the ``behavior" of a dimensionality reduction method, namely,
whether it tends to produce an ``intrusive" ($B_{nx}(K)>0$)
or ``extrusive" ($B_{nx}(K)<0$) embedding.
Intrusive embedding tends to crush the manifold, which means faraway points
can become neighbors after embedding, while extrusive one tends to tear the manifold,
meaning some close neighbors can be embedded faraway from each other.
In an ideal projection, $B_{nx}$ should be zero.
See \cite{LeeV08jmlr} for more details.

    The comparison of LLE, Isomap, MVU and our proposed  method on the
    teapot (with $\sigma = 10^{10}$) and the swiss roll ($\sigma =
    10^{5}$) datasets are shown in Figure~\ref{fig:QB_S500} and
    Figure~\ref{fig:QB_T200}. As can be seen from
    Figure~\ref{fig:QB_S500} and Figure~\ref{fig:QB_T200} (a), the
    proposed FrobMetric method performs on par with MVU, while better
    than both Isomap and LLE in terms of neighborhood preservation.
    Note that all methods tend to tear the manifold as $B_{nx}(K)$ is
    below zero in all cases.

\begin{figure}[htb!]
\begin{center}
\begin{tabular}{c}
\includegraphics[width=0.4\textwidth]{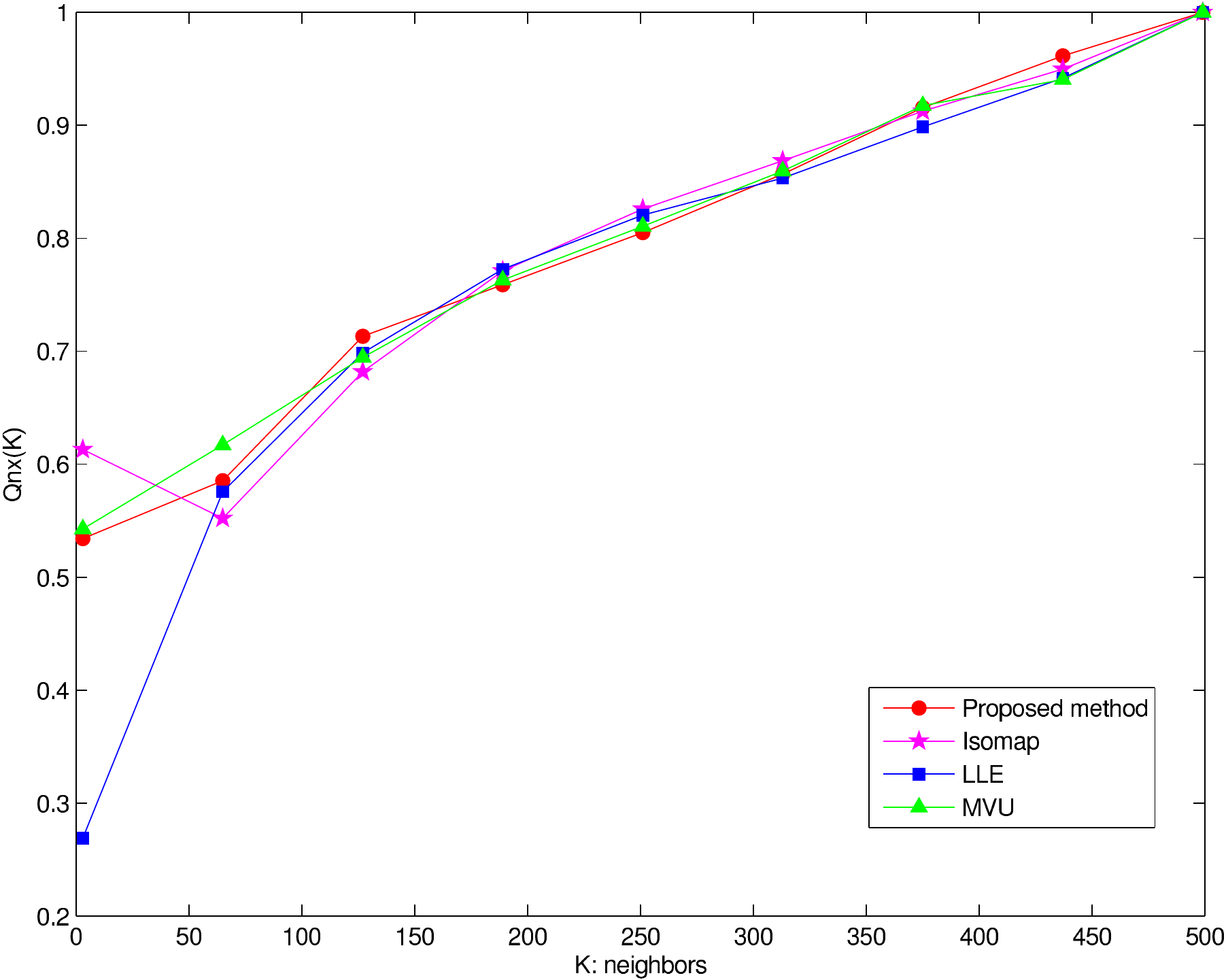}\\(a)
\\
\includegraphics[width=0.4\textwidth]{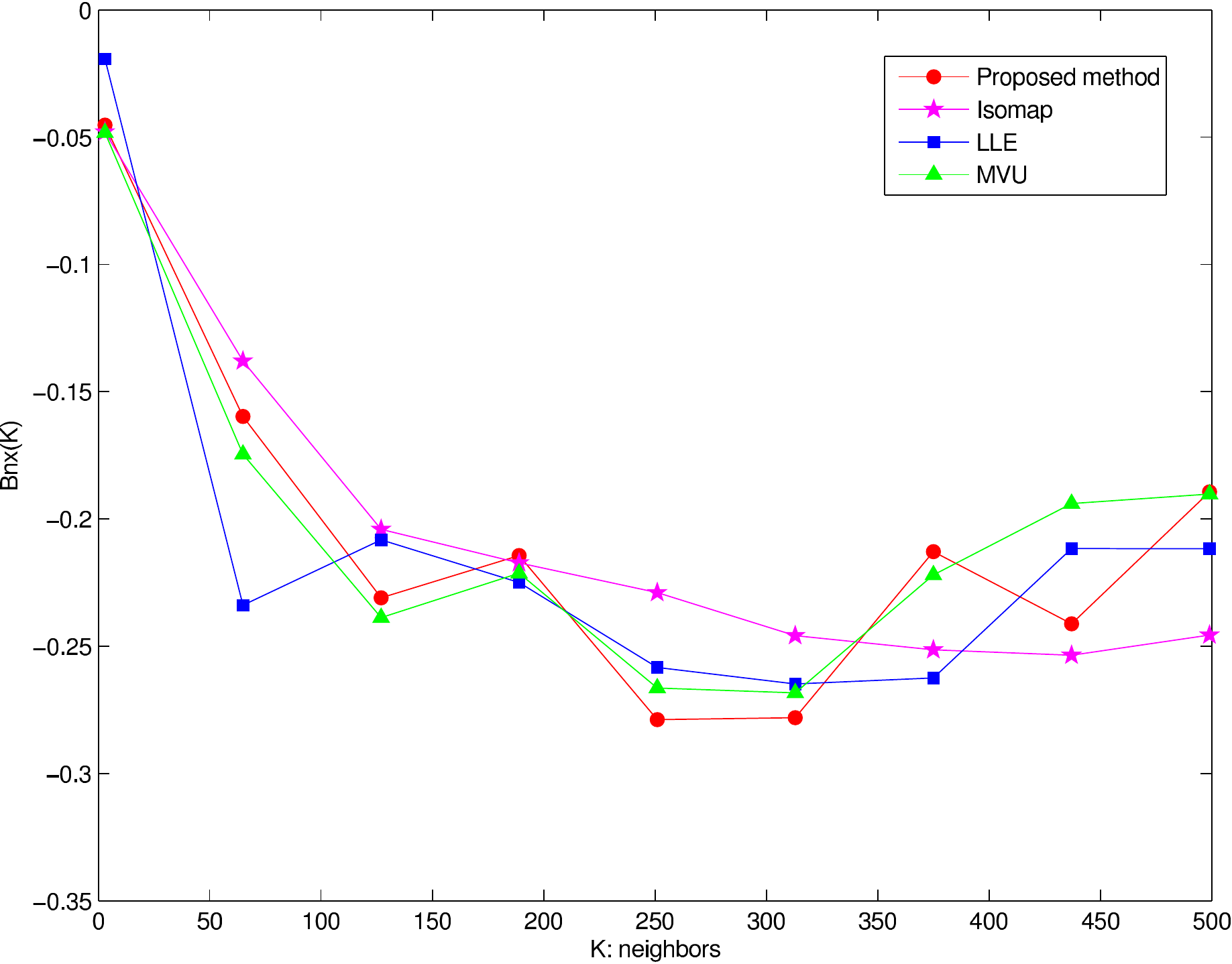}\\(b)
\end{tabular}
\end{center}
\caption[Quality assessment of neighborhood preservation of different
algorithms on the
3D swiss-roll data
]{Quality assessment of neighborhood preservation of different algorithms on 3D swiss roll;
(a) $Q_{nx}(K)$;(b) $B_{nx}(K)$.} \label{fig:QB_S500}
\end{figure}

\begin{figure}[htb!]
\begin{center}
\begin{tabular}{c}
\includegraphics[width=0.4\textwidth]{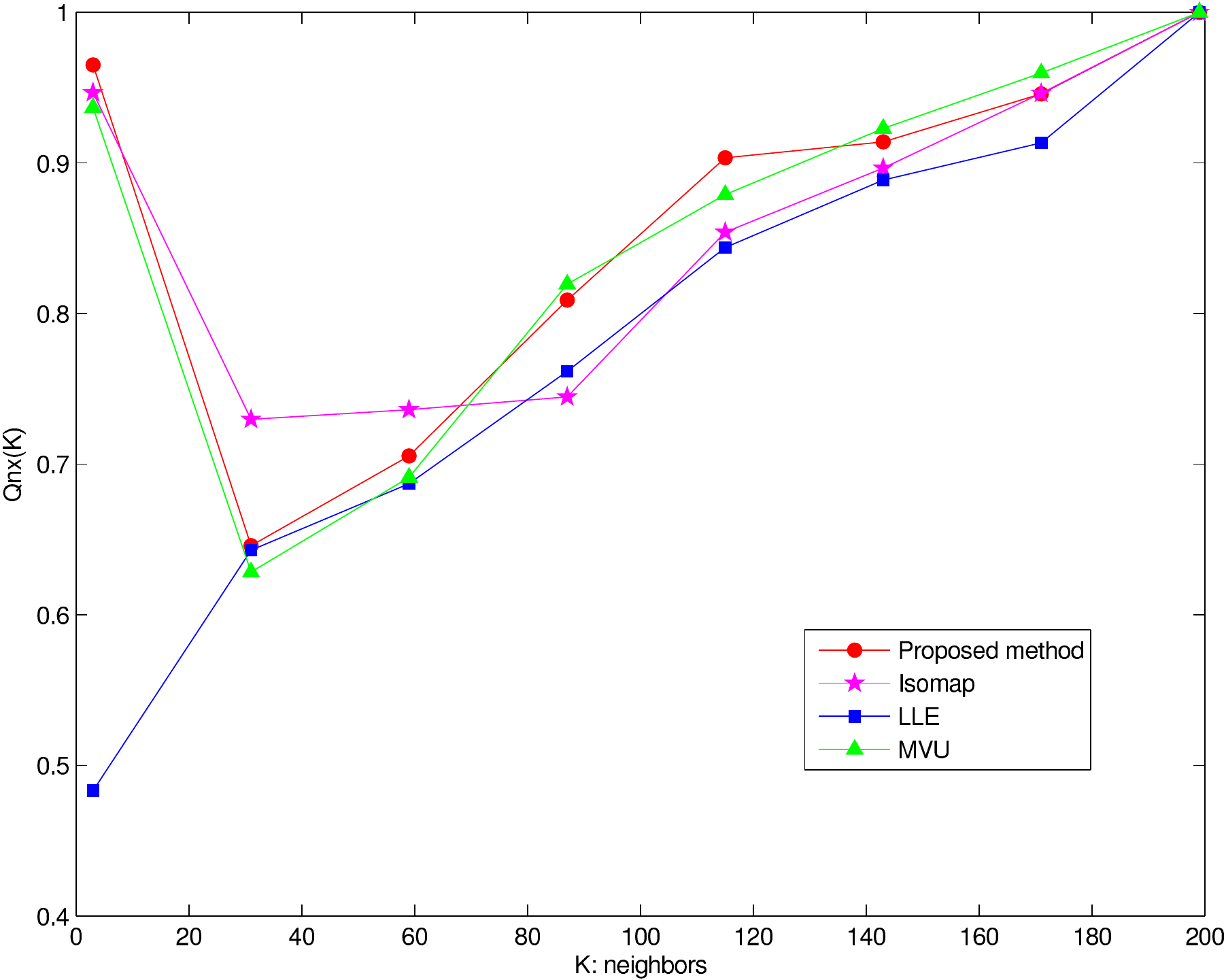}\\(a)
\\
\includegraphics[width=0.4\textwidth]{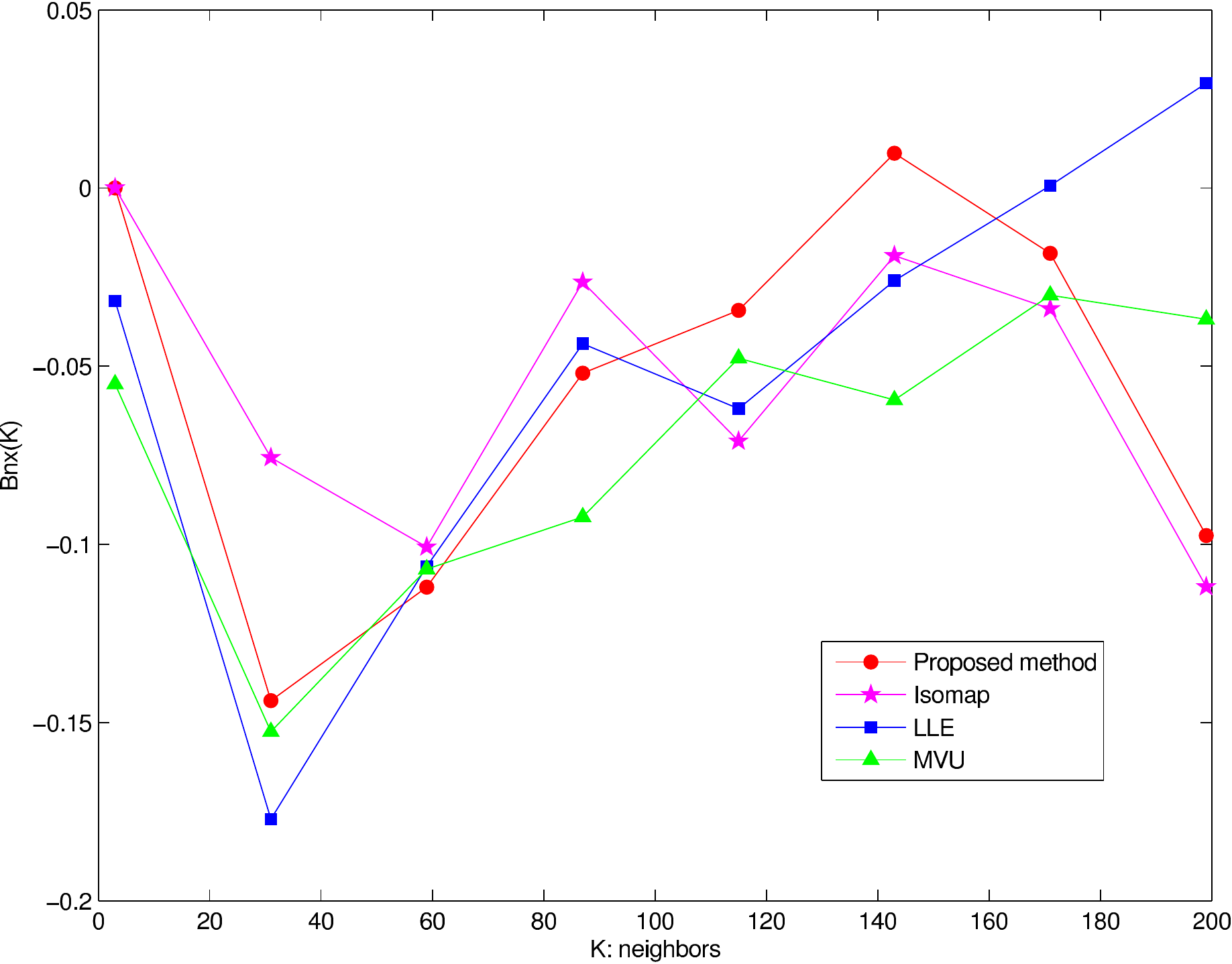}\\(b)
\end{tabular}
\end{center}
\caption[Quality assessment of neighborhood preservation of different algorithms on teapot
dataset]{Quality assessment of neighborhood preservation of
different algorithms on teapot dataset; (a)
$Q_{nx}(K)$;
(b) $B_{nx}(K)$.} \label{fig:QB_T200}
\end{figure}

We have also made quantitative analysis of the proposed algorithm based on Zhang \etal~\cite{ZhangWHZ11TNN}.
They proposed several quantitative criteria, specifically, average local standard deviation
(ALSTD)
and average local extreme deviation (ALED) to measure the global smoothness of a recovered
low-dimensional manifold; average local co-directional consistence (ALCD) to estimate the
average co-directional consistence of the principle spread direction (PSD) of the data points,
and a combined criteria to simultaneously evaluate the global smoothness and co-directional
consistence (GSCD).

We give the visual results of swiss-roll dataset based on PSD in Figure~\ref{fig:swissQA},
in which the longer line at each sample represents the first PSD,
and the second line is orthogonal to the first PSD.
We also report the ALSTD and ALED, ALCD and GSCD in Table \ref{tab:alResult}
and Table \ref{tab:cdResult}. From the tables we see that MVU performs best
on this swiss-roll dataset, while the proposed FrobMetric method ranks
the second best.
On the teapot dataset,
the proposed method performs slightly better than MVU, while worse than both Isomap and LLE.
Overall, the proposed method is similar to the original MVU in terms
of these embedding quality criteria. However, note that the proposed
method is much faster than MVU in all cases.
\begin{figure}[htb!]
    \centering
    \begin{tabular}{c}
    \includegraphics[width=0.3\textwidth]{./mvu/swiss_data} \\
    original data
    \end{tabular}
    \begin{tabular}{cc}
    \includegraphics[width=0.24\textwidth]{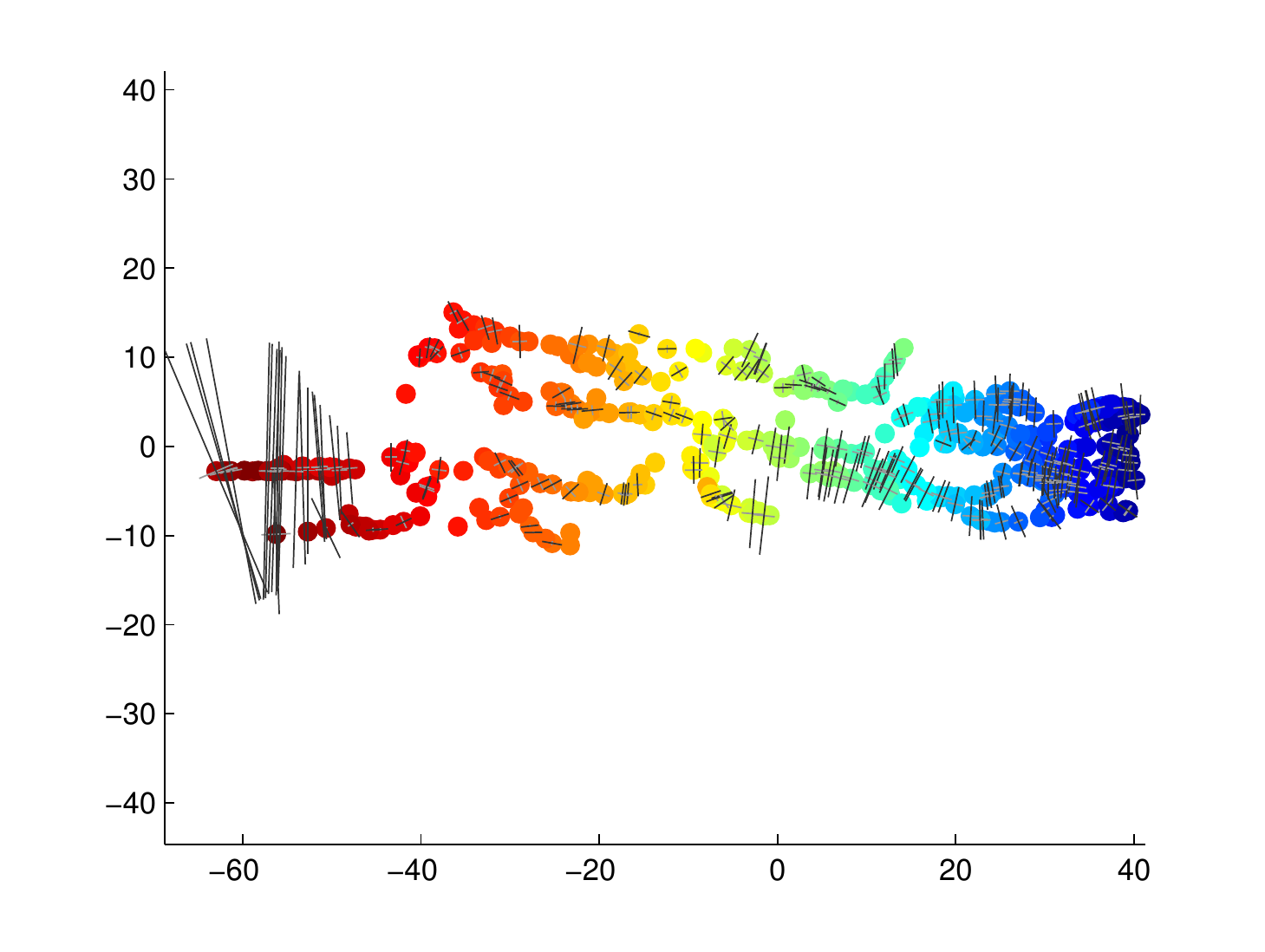} &
    \includegraphics[width=0.24\textwidth]{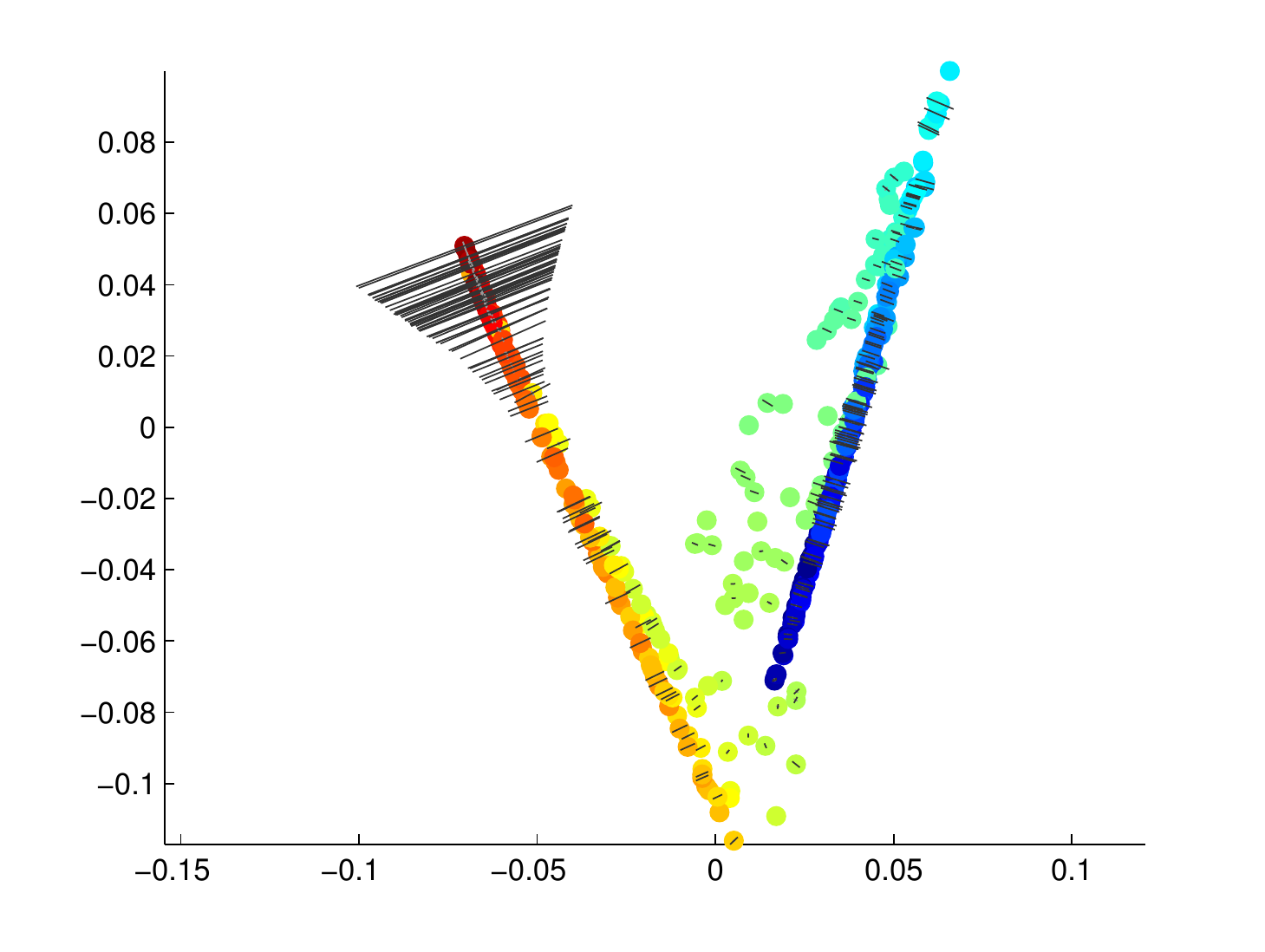}
    \\ (a) & (b)
    \end{tabular}
    \begin{tabular}{cc}
    \includegraphics[width=0.24\textwidth]{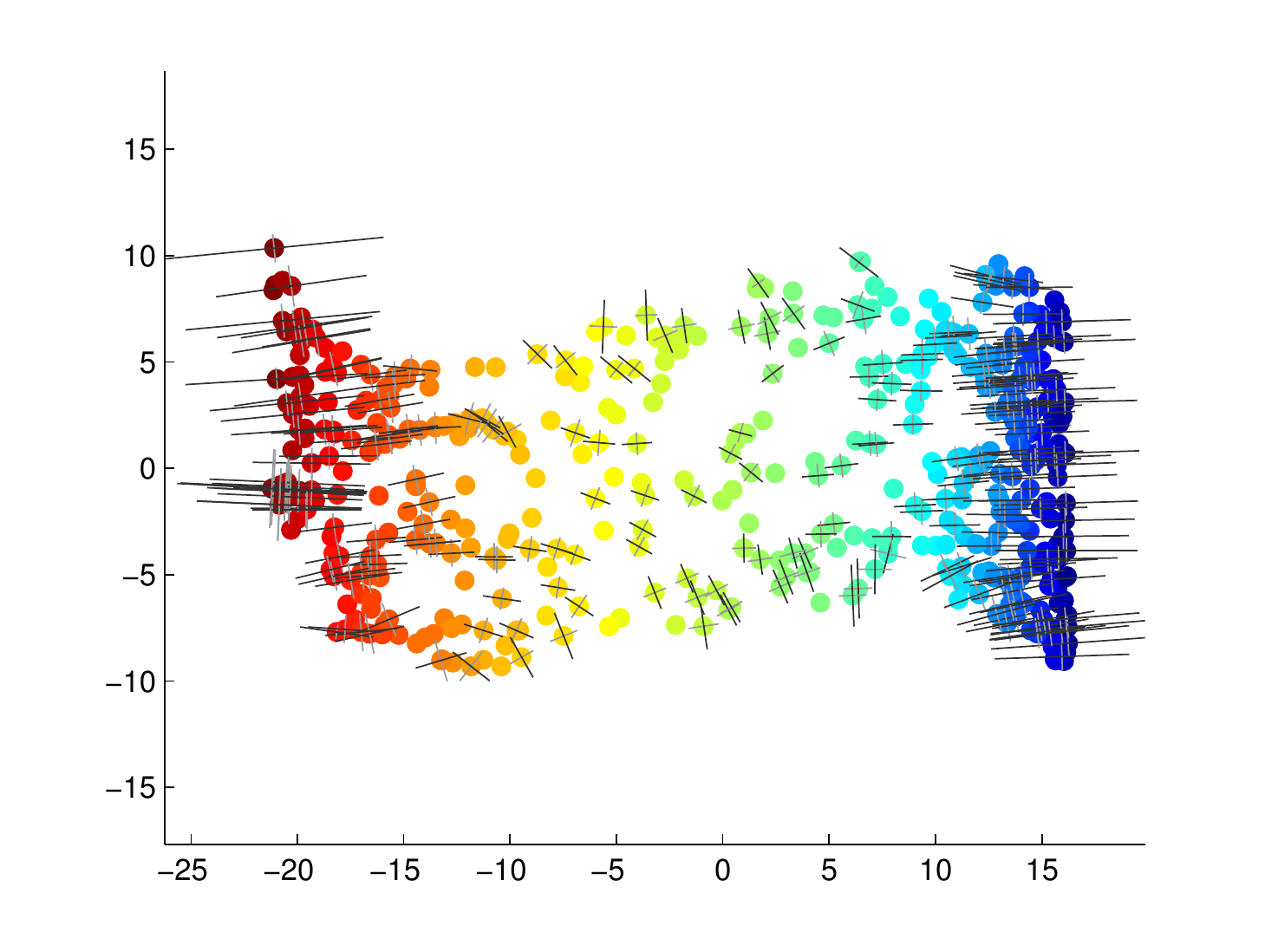} &
    \includegraphics[width=0.24\textwidth]{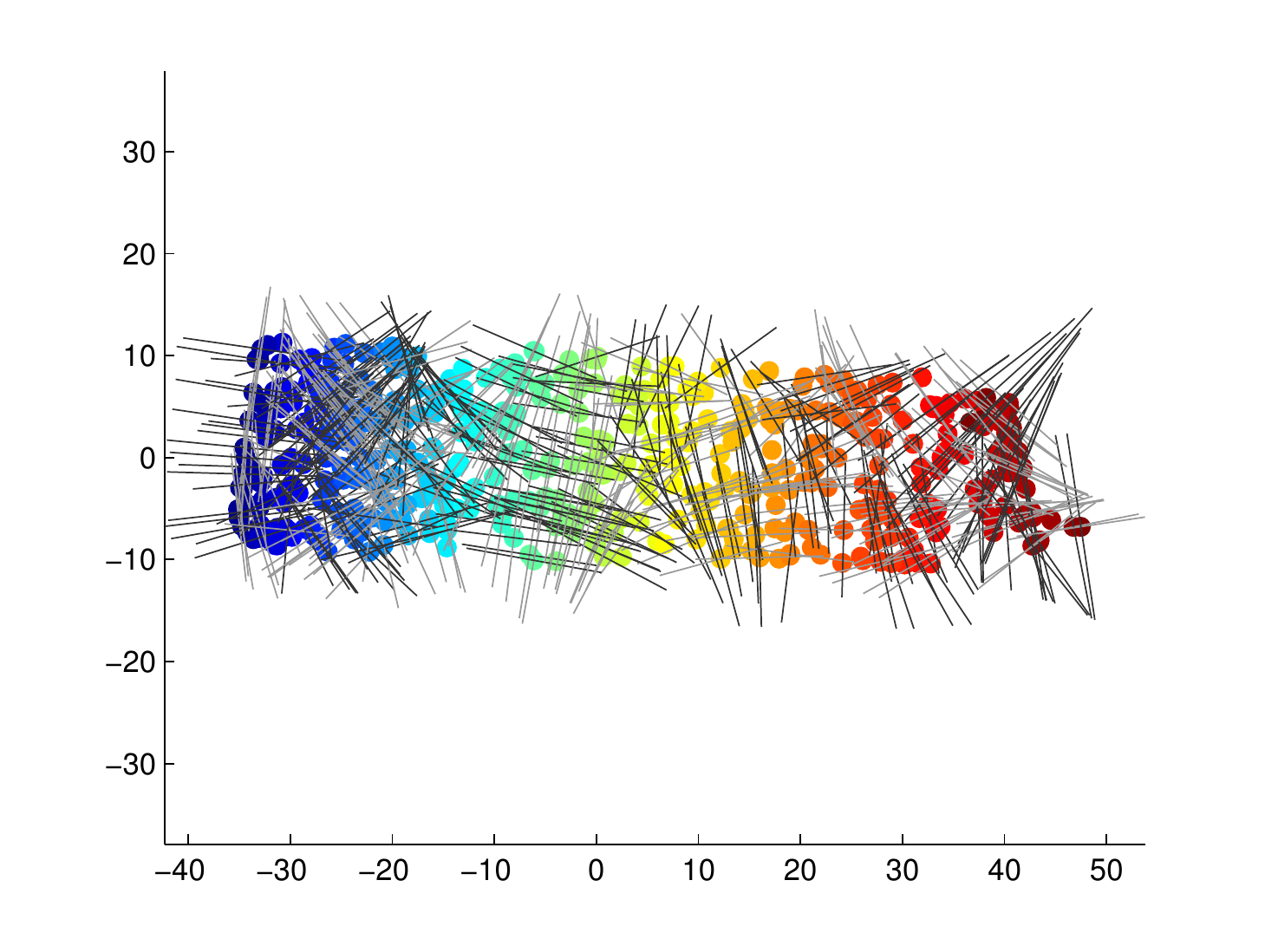}
    \\ (c) & (d)
    \end{tabular}
    \caption{Visualization results based on PSD on 3D swiss roll,
    with the neighborhood size $k=6$ for all,
    and $\sigma = 10^{5}$ for our method; (a)Isomap, (b)LLE,
    (c) our method, (d) MVU.} \label{fig:swissQA}
\end{figure}

\begin{table}[!hbp]
\centering
\caption{ALSTD and ALED results on the 3D swiss roll and teapot datasets} \label{tab:alResult}
\begin{tabular}{|c|c|c|c|c|}
    \hline
    Algorithm &  FrobMetric & MVU & Isomap & LLE \\  \hline
    \multicolumn{5}{|c|}{ALSTD} \\  \hline
    Swiss roll & 0.113 & \textbf{0.038} & 0.245 & 0.328\\  \hline
    Teapot & 0.377 & 0.481 & \textbf{0.0611} & 0.0805\\  \hline
    \multicolumn{5}{|c|}{ALED} \\  \hline
    Swiss roll & 0.311 & \textbf{0.102} & 0.668 & 0.886\\ \hline
    Teapot & 1.0  & 1.28 & \textbf{0.173} & 0.232\\
    \hline
\end{tabular}
\end{table}

\begin{table}
\centering
\caption{ALCD and GSCD results on the 3D swiss roll and teapot datasets} \label{tab:cdResult}
\begin{tabular}{|c|c|c|c|c|}
    \hline
    Algorithm &  FrobMetric & MVU & Isomap & LLE \\ \hline
    \multicolumn{5}{|c|}{ALCD} \\ \hline
    Swiss roll & 0.983 & 0.964 & 0.969 & \textbf{0.994}\\  \hline
    Teapot & 0.995 & 0.942 & \textbf{0.997} & 0.995\\ \hline
    \multicolumn{5}{|c|}{GSCD} \\  \hline
    Swiss roll & 0.1163 & \textbf{0.0404} & 0.2572 & 0.3317\\ \hline
    Teapot & 0.3792 & 0.5105 & \textbf{0.0613} & 0.0808\\
    \hline
\end{tabular}
\end{table}

To show the efficiency of our approach,
in Figure~\ref{fig:mvu_time} we have compared the computational time
between the original MVU implementation and the proposed method,
by varying the number of data samples, which determines the number of variables in MVU.
Note that the original MVU implementation uses CSDP \cite{borchers1999csdp}, which is an interior-point
based Newton algorithm.
We use the 3D ``swiss-roll'' data here.
\begin{figure}[htb!]
\begin{center}
\begin{tabular}{c}
\includegraphics[width=0.45\textwidth]{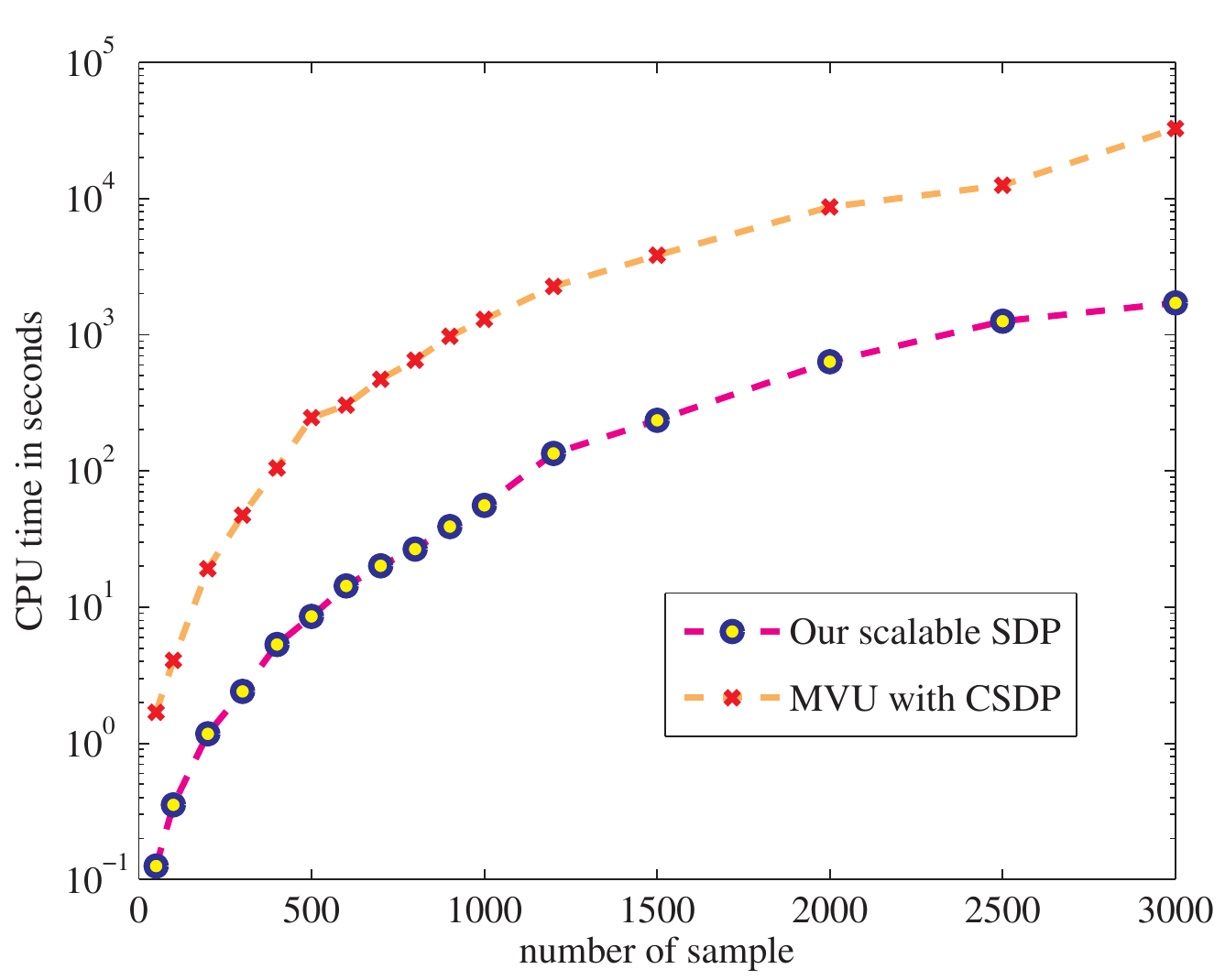}
\end{tabular}
\end{center}
\caption{ {Comparison of computational time on 3D Swiss roll dataset
                  between MVU and our fast approach.
                  Our algorithm uses $\sigma = 10^2$.
                  Our algorithm is about $ 15 $ times faster.
                  Note that the $ y$-axis is in log-scale.}
        } \label{fig:mvu_time}
\end{figure}

}  %

\section{Conclusion}
\label{Conclusion}

    We have presented an efficient and scalable semidefinite metric learning algorithm.
    Our algorithm is simple to implement and much more scalable than most SDP solvers.
    The key observation is that, instead of solving the original primal problem, we solve
    the Lagrange dual problem by exploiting its special structure. Experiments on UCI benchmark
    data sets as well as the unconstrained face recognition task show its efficiency and efficacy.
    We have also extended it to solve more general Frobenius norm regularized SDPs.

\bibliographystyle{ieee}

\end{document}